\documentclass[twoside,11pt,jmlr]{article}

\newcommand{\fracpartial}[2]{\frac{\partial #1}{\partial  #2}}
\usepackage{xcolor}
\usepackage{bbm}
\usepackage{mathtools}

\usepackage{amsmath,amsfonts,bm}

\def\eqref#1{equation~\ref{#1}}

\def\1{\bm{1}}

\def\vu{{\bm{u}}}
\def\vv{{\bm{v}}}
\def\vw{{\bm{w}}}

\DeclareMathAlphabet{\mathsfit}{\encodingdefault}{\sfdefault}{m}{sl}
\SetMathAlphabet{\mathsfit}{bold}{\encodingdefault}{\sfdefault}{bx}{n}

\def\gA{{\mathcal{A}}}

\def\gD{{\mathcal{D}}}

\def\gG{{\mathcal{G}}}
\def\gH{{\mathcal{H}}}

\def\gL{{\mathcal{L}}}
\def\gM{{\mathcal{M}}}

\def\gP{{\mathcal{P}}}

\def\gX{{\mathcal{X}}}

\def\sR{{\mathbb{R}}}

\newcommand{\E}{\mathbb{E}}

\newcommand{\Var}{\mathrm{Var}}

\DeclareMathOperator*{\argmax}{arg\,max}

\renewcommand{\epsilon}{\varepsilon}

\newcommand{\A}{\mathcal{A}}

\renewcommand{\Pr}{\textup{\mbox{Pr}}}

 \DeclareMathOperator*{\tr}{tr}

\def\ie{\textit{i.e.,~}}

\def\wrt{\textit{w.r.t.~}}

\def\ones{\mathbbm{1}}
\def\Trans{^\top}

\usepackage{enumitem}
\usepackage{wrapfig}

\usepackage{enumitem}
\usepackage{wrapfig}

\usepackage{algorithm} 
\usepackage{algpseudocode} 
\usepackage{booktabs} %
\usepackage{threeparttable}  %
\usepackage{multirow} %
\usepackage{graphicx}
\usepackage{float}
\usepackage{color}
\usepackage{enumitem}
\usepackage[hang]{subfigure}
\usepackage{amsmath} %
\usepackage{colortbl} %

\usepackage{url}
\usepackage{amsthm}

\usepackage{jmlr2e}

\newtheorem{prop}[theorem]{Proposition}
\newtheorem{assumption}[theorem]{Assumption}

\theoremstyle{definition}

\newcommand{\paraemph}[1]{\textit{#1}} %

\ShortHeadings{A Theoretical Framework for Masked Pretraining (MPT)}{Zhang, Zhou, Wang and Wang}
\firstpageno{1}

\usepackage{hyperref}

\usepackage{lastpage}
\jmlrheading{27}{2026}{1-\pageref{LastPage}}{3/25; Revised 8/26}{8/26}{25-0477}{Qi Zhang, Runyu Zhou, Yifei Wang and Yisen Wang}

\begin{document}

\title{A Theoretical Framework for Masked Pretraining (MPT)}

\author{\name Qi Zhang $^1$ $^*$
\email zhangq327@stu.pku.edu.cn 
\AND
\name Runyu Zhou $^1$ $^*$
\email runyuzhou25@stu.pku.edu.cn
\AND
\name Yifei Wang $^2$ \thanks{Equal Contribution}
\email yifei\_w@mit.edu
\AND
\name Yisen Wang $^{1}$ \thanks{Corresponding Author}
\email yisen.wang@pku.edu.cn \\~\\
\addr $^1$State Key Lab of General Artificial Intelligence, School of Intelligence Science and Technology, Peking University \\
\addr $^2$CSAIL, MIT
}
\editor{Sanjiv Kumar}
\maketitle

\begin{abstract}
    Recently, Masked Pretraining (MPT) based on reconstruction pretraining tasks has risen to a promising self-supervised learning paradigm across various domains and achieves remarkable performance in multiple downstream tasks. However, the theoretical understanding of the working mechanism behind MPT is still limited. In this paper, we introduce a new theoretical framework to analyze MPT and understand the crucial role of masking in extracting meaningful representations. We establish theoretical connections between MPT and another popular self-supervised paradigm: contrastive learning. We prove that the masking technique implicitly creates positive pairs that are semantically similar and the reconstruction loss pulls them together in the feature space. Besides, as a result of the implicit alignment, we point out the dimensional collapse issue of MPT and propose a Uniformity-enhanced MPT (U-MPT) loss that can effectively address this issue and bring significant improvements in downstream tasks including linear evaluation, cross-dataset fine-tuning and out-of-distribution generalization on real-world data sets. Furthermore, we establish downstream guarantees of U-MPT and theoretically analyze the influence of masking strategies. Based on the theoretical analysis, we propose a new masking strategy which enhances the downstream performance of MPT and explains current improvements of masking strategies with our theoretical perspective.
\end{abstract}


\section{Introduction}

In recent years, self-supervised learning (SSL) has emerged as a promising paradigm for learning meaningful representations without the need for labeled data. SSL methods can be broadly categorized into two types: contrastive pretraining and masked pretraining. Contrastive pretraining methods \citep{simclr,moco} apply data augmentation on natural samples and train neural networks by aligning different augmented views of the same sample in the feature space. On the other hand, masked pretraining methods \citep{bert,mae} apply masks to samples and predict the masked portions based on the unmasked views. Although contrastive and masked pretraining differ in their objectives, both approaches have demonstrated impressive performance across various domains, including vision \citep{mae,simclr}, language \citep{bert,zhang2022contrastive}, and multi-modal tasks \citep{clip,khan2021exploiting,zhang2023on}.

Due to the remarkable empirical success of self-supervised learning, researchers try to theoretically analyze the working mechanism behind these paradigms. Various theoretical frameworks have been proposed to analyze contrastive pretraining from different perspectives. For example, \citet{arora} establish the connection between the pretraining contrastive objective and downstream classification error. \citet{haochen} observe that the contrastive loss is equal to a matrix decomposition objective, and \citet{infomax} understand the contrastive pretraining process with the information theory. However, the theoretical analysis of masked pretraining is still relatively underexplored.

In this paper, we introduce a new theoretical framework for masked pretraining (MPT). Taking the representative masked pretraining method Masked Autoencoder (MAE) \citep{mae} as an example, we note that there are two core components: autoencoders and masking. As vanilla autoencoders exhibit much inferior performance than contrastive paradigms \citep{simclr}, there exists a common belief that masking is the key factor for the success of MPT. To understand the role of masking in MPT, we first establish a close connection between contrastive and masked pretraining. Specifically, we prove that the masking technique creates implicit positive pairs that are semantically similar and the reconstruction loss pulls these positive pairs together.
In contrastive pretraining, merely aligning positive pairs can lead to feature collapse, as the alignment loss can be minimized by mapping different inputs to a constant value. Given the connection between contrastive and masked pretraining, we wonder whether MPT also faces this collapse issue. Our findings show that while MPT avoids the trivial constant solution, it still suffers from dimensional collapse, where the effective dimensions of the learned representations are quite limited. To address this problem, we draw inspiration from contrastive learning and add a uniformity regularizer into the MPT objective. Specifically, we propose Uniformity-enhanced MPT (U-MPT), which calculates the feature similarity between independent samples and encourages them to be pushed apart. Empirically, we observe that U-MPT effectively resolves the dimensional collapse and significantly boosts MPT performance on downstream linear evaluation, cross-dataset fine-tuning and out-of-distribution generalization tasks.

Built on the connection between contrastive pretraining and masked pretraining, we then establish theoretical guarantees for U-MPT's downstream generalization performance. Our theoretical bounds demonstrate that the generalization performance of MPT is decided by two factors: first, the unmasked views should preserve the core semantics of natural samples (like label information), and second, semantically similar unmasked views should be selected as an implicit positive pair. With a toy model, we prove that a lower mask ratio helps the unmasked views retain label information while a higher mask ratio increases the probability that semantically similar unmasked views are selected as implicit positive pairs. This introduces a trade-off in selecting the optimal mask ratio. Based on this insight, we propose a surrogate metric to estimate downstream guarantees on real-world data sets, and we confirm that it can reliably estimate the optimal mask ratio in practice (75\%). Furthermore, with the theoretical analysis of masking operation in MPT, we propose an improved masking strategy, which exhibits better downstream classification performance than random masking. Meanwhile, we observe that many of the recent advancements in masking strategies and objective designs can be analyzed and explained through our theoretical framework, which further verifies the generality of our proposed analysis. The contents of the paper are organized as follows:

\begin{itemize}
    \item In Section 2, we summarize the related work of understanding self-supervised learning and further introduce the mathematical problem setup in Section 3.
    \item In Section 4, we establish theoretical connections between masked pretraining and contrastive pretraining, showing that the masking operation implicitly creates positive pairs that are semantically similar and the reconstruction objectives align them in the feature space.
    \item  In Section 5, we propose the uniformity-enhanced MPT (U-MPT) objective, which alleviates the dimensional collapse issue in MPT and significantly improves the downstream performance of MPT in downstream tasks including linear evaluation, cross-dataset fine-tuning, and out-of-distribution generalization on real-world data sets.
    \item In Section 6, we establish theoretical guarantees on the downstream performance of MPT. With a toy model, we theoretically analyze the influence of mask ratio on the downstream error of MPT. Furthermore, we propose a metric to estimate MPT's downstream performance, which accurately predicts the optimal mask ratio.
    \item In Section 7, we improve the masking strategy in MPT based on our theoretical analysis. Besides, we also find that many of the current improvements can be explained with our theoretical framework.
\end{itemize}

\paraemph{Remark.} A preliminary work was published at NeurIPS 2022 as a Spotlight paper \citep{zhang2022mask}, while we add substantially new results, both theoretically and empirically, in this manuscript, aiming to provide a more comprehensive and in-depth analysis of the generalization property of masked pretraining. Specifically, we add the following key results:

\begin{itemize}
    \item In Sections 3 and 4, we incorporate cross-entropy (CE) loss-based MPT methods into our theoretical framework alongside the reconstruction loss-based approaches, thereby enhancing the comprehensiveness of our analysis.
    \item  In Section 5, we verify our proposed U-MPT objectives in more downstream scenarios, including cross-dataset fine-tuning and out-of-distribution generalization, which further verify the effectiveness and advantages of U-MPT.
    \item In Section 6, we theoretically analyze the influence of mask ratio on downstream guarantees with a toy model. We also propose a new metric to estimate the theoretical bounds.
    \item In Section 7, we propose new masking and reconstruction strategies based on our theoretical analysis. Besides, we explain the current improvements of MPT with our theoretical framework.
\end{itemize}

\section{Related Work}

In this section, we review two lines of work most relevant to our study: the empirical development and the theoretical efforts of masked pretraining (MPT) methods.

\paraemph{Masked Pretraining Methods in Practice.} Deep supervised learning algorithms typically require extensive labeled data to achieve strong performance, yet manual annotation is costly and labor-intensive. Self-supervised learning (SSL) addresses this challenge by learning discriminative representations directly from unlabeled data. Among SSL paradigms, contrastive learning (CL) and Masked Pretraining (MPT) approaches have emerged as dominant frameworks. CL trains models to pull semantically similar samples together in the feature space while pushing dissimilar samples away, achieving notable success in vision and language tasks~\citep{simclr,moco,BYOL,wang2021residual,simsiam,barlowtwins,vicreg,cui2023rethinking,wang2023message}. MPT methods based on reconstruction tasks have also demonstrated impressive performance \citep{bert,mae,du2024on}. During the pretraining process, MPT methods reconstruct masked regions of input data to learn transferable representations. Key MPT methods include BERT~\citep{bert}, BeiT~\citep{beit}, SimMIM~\citep{xie2022simmim}, and MAE~\citep{mae}, which differ in reconstruction targets: BEiT and BERT optimize token-level cross-entropy losses, while SimMIM and MAE minimize pixel-wise losses. 

\paraemph{Theoretical Understandings of Masked Pretraining Methods.}  The empirical success of SSL has made it important to theoretically understand its underlying mechanisms. For contrastive learning,  \citet{arora} establish the downstream guarantee of contrastive learning representations.  \citet{wang2022chaos,zhangjmlr2025} revise their bounds by resolving the class collision issues, and develop a new understanding of contrastive learning from the perspective of augmentation overlap. 
\citet{wang2020understanding} identify alignment and uniformity as two key properties of contrastive learning.
\citet{haochen} introduce the augmentation graph framework to characterize the downstream performance of spectral contrastive solutions. \citet{cui2025augmentation} propose an augmentation-aware error bound which enables the explanation of specific types of data augmentation in contrastive learning. \citet{zhang2023on} establish generalization guarantees for multimodal contrastive learning. 
However, compared to contrastive learning, the theoretical analysis of MPT methods remains relatively limited. \citet{cao2022understand} analyze the attention mechanism of MPT through an integral kernel perspective. \cite{Kong2023UnderstandingMAE} formulate the underlying data-generating process in MPT as a hierarchical latent variable model.
\citet{liu2022masked} study masked prediction from a parameter identifiability perspective and characterize when the underlying parameters can be recovered from masked prediction objectives. \citet{meng2024representation} reveal representation deficiency in masked language modeling and theoretically show that the mask token can occupy representation capacity, leading to rank deficiency in learned representations. \citet{li2024promises} establish a theoretical framework for generative masked language modeling and analyze its statistical properties, including the impact of masking ratios on sample complexity and learning efficiency. \citet{zhang2024look} theoretically compare masked and autoregressive pretraining, showing that the flexible prediction targets in masked pretraining induce richer inter-sample connections and lead to advantages in downstream classification.
However, most of these studies predominantly analyze model architecture and ignore the core designs of MPT methods, i.e., the masking technique. Our work bridges this gap by formalizing the role of masking and deriving the first theoretical guarantees for MPT’s downstream performance. We note that while both \citet{haochen} and our work adopt a graph-based perspective, the corresponding graphs characterize different learning mechanisms. Their augmentation graph defines edges through data augmentations and is developed for analyzing contrastive learning, whereas our mask graph defines edges through reconstruction relationships between masked and unmasked views and is used to characterize the implicit alignment induced by masked reconstruction.

\section{Mathematical Formulation for Different Masked Pretraining Methods}

In this section, we introduce mathematical formulations of  Masked Pretraining (MPT) approaches. We categorize different MPT approaches based on the type of loss used, i.e., the pixel-wise reconstruction loss and the cross-entropy loss. We start by introducing some common definitions for clarity.

Let \(\bar{x}\) denote a natural sample from a data set \(\gD\). In computer vision (CV), \(\bar{x}\) typically refers to an image, while in natural language processing (NLP), it corresponds to a sequence of words. The sequence \(\bar{x}\) is divided into \(K\) elements, where each element is either a patch of size \(s\) (e.g., \(16 \times 16\) in ViT) in CV, such that \(\bar{x} \in \sR^{K \times s}\), or a token in NLP. To generate a masked view, a random binary mask \(m \in \{0, 1\}^K\) is applied, where each entry \(m_k\) is set to \(0\) with probability \(\rho\) (the mask ratio). This process masks a fraction \(\rho\) of the elements, resulting in two complementary views of the sequence: the unmasked view \(x_1\) and the masked view \(x_2\). These views are formally defined as:
\begin{equation}
\label{eq:mask_view} 
x_1 = \bar{x}[m], \quad x_2 = \bar{x}[1 - m],
\end{equation}
where \(\bar{x}[m]\) selects the elements of \(\bar{x}\) corresponding to the positions where \(m_k = 1\). Let \(n_1\) and \(n_2\) denote the number of unmasked and masked elements. By construction, \(n_1 + n_2 = K\), with \(n_1 = (1-\rho)K\) and \(n_2 = \rho K\). 

\paraemph{The Mask Graph $\gG_M$ of MPT.} We notice that MPT essentially learns to pair the unmasked and masked views $x_1, x_2$ via the reconstruction task, which can be characterized by a \emph{mask graph} $\gG_M$. Denote the set of all unmasked views as $\gX_1=\{x_1\}$ and the set of all masked views as $\gX_2=\{x_2\}$, where the two sets are assumed to be finite (can be exponentially large), \ie $|\gX_1|=N_1,|\gX_2|=N_2$. The mask graph $\gG_M$ over the joint set $\gX=\gX_1\cup\gX_2$ is defined here:
\begin{itemize}
    \item Node: each view $x\in\gX$.
    \item Edge: the edge weight $w_{x_1,x_2}$ between any $x_1,x_2\in\gX$ is defined as their joint probability $\gM(x_1,x_2)=\E_{\bar x}\gM(x_1,x_2|\bar x)$. In other words, there is an edge between two views (\ie $w_{x_1,x_2}>0$) if and only if they are complementary views generated by masking.
\end{itemize}
Considering the masking process in Eq. \ref{eq:mask_view}, when the masking ratio $\rho\neq 0.5$, the two sets $\gX_1$ and $\gX_2$ are disjoint by construction. Therefore, there only exist edges \emph{between} the two sets and there is no edge \emph{within} either set. Thus, the mask graph $\gG_M$ is a \emph{bipartite} graph in this case.
Its adjacency matrix can be defined as $A_M\in\sR^{N_2 \times N_1}$ where $(A_M)_{x_2,x_1}=w_{x_1,x_2}$ for $x_1\in\gX_1,x_2\in\gX_2$. As long as $\rho\neq0.5$, $N_1\neq N_2$ and $A_M$ is not necessarily a square matrix. 
We can define the normalized adjacency matrix as $\bar A_M=D_2^{-1/2}AD_1^{-1/2}$, where  $D_1,D_2$ are the diagonal degree matrices with elements $d_{x_1}=\sum_{x_2}w_{x_1,x_2}$ and $d_{x_2}=\sum_{x_1}w_{x_1,x_2}$, respectively.

\subsection{Methods with Pixel-wise Reconstruction Loss}

A wide range of MPT methods in CV is based on pixel-wise reconstruction loss, which measures differences between the original images and the reconstructed regions \citep{mae,xie2022simmim}. For instance, MAE~\citep{mae} employs an encoder-decoder architecture $h = g \circ f$, where $f$ is the encoder that maps unmasked patches $x_1$ to latent features $z_1$, and $g$ is the decoder that reconstructs the complementary masked view $x_2$ by mapping $z_1$ and the mask token back to the pixel space. During the pretraining process, MAE minimizes the $l_2$-norm loss between the original and reconstructed masked areas:
\begin{equation}
    \gL_\text{MAE}(h)=\E_{\bar x}\E_{x_1,x_2|\bar x}\left\| g(f(x_1))-\hat{x}_2\right\|_2^2,
    \label{eq:mae-loss}
\end{equation}
where $\hat{x}_2$ is the $l_2$-normalized version of $x_2$.

SimMIM~\citep{xie2022simmim} also uses an encoder-decoder framework, where the encoder $f$ takes the unmasked view $x_1$ and a mask token (spread over the masked positions) as input, outputting the corresponding features. A linear layer $g$ is then used to reconstruct the masked patches from these features, and the loss function is defined as an $l_1$ reconstruction loss:

\begin{equation}
    \gL_\text{SimMIM}(h)=\frac{1}{n_2}\E_{\bar x}\E_{x_1,x_2|\bar x}\left\| g(f(x_1))-x_2\right\|_1.
    \label{eq:simmim-loss}
\end{equation}

\subsection{Methods with Cross-Entropy Loss}

While many MPT methods in the vision domain rely on the pixel-wise reconstruction loss (such as \( l_1 \) or \( l_2 \) loss), MPT pretraining approaches usually adopt the cross-entropy (CE) loss in the language domains \citep{bert}. Taking the popular framework BERT \citep{bert} as an example, the pretraining objective is to minimize the negative log-likelihood:
\begin{equation}
\gL_{\text{BERT}} = -\E_{x_1,x_2}\sum_k \log\Pr(x_{2,k}^+ \mid x_1),
\end{equation}
where \( x_1 \) and \( x_2 \) are the unmasked and masked portions of a sequence from a language data set, \( x_{2,k}^+ \) denotes the \( k \)-th token of the masked view \( x_2 \). This formulation captures the likelihood of predicting the masked tokens in \( x_2 \) based on the unmasked portion \( x_1 \). Specifically, the predicted probability of a token \( x_{2,k}^+ \) in the masked portion can be formulated as:
\begin{align}
\gL_{\text{BERT}} &=  -\E_{x_1,x_2}\sum_k \log \frac{\exp(f(x_1)^\top w_{x_{2,k}^+})}{\sum_{x_{2,k}^-}\exp(f(x_1)^\top w_{x_{2,k}^-})},
\label{eqn:ce-softmax}
\end{align}
where \( x_{2,k}^- \) denotes irrelevant tokens in the vocabulary, \( f(x_1) \) is the encoder output given the input \( x_1 \), and \( w_i \) is the feature embedding of the \( i \)-th token in the weight matrix of the token classification head.

Inspired by the impressive performance of BERT in the language domain, \citet{beit} adapts the concept of masked token prediction to vision tasks. Unlike methods that operate directly on raw pixels, BEiT employs a discrete variational autoencoder (dVAE) to transform an image \(\bar{x}\) into a sequence of discrete tokens. Specifically, each patch of the image is tokenized into an index from a visual codebook. BEiT then learns to predict the token indices of the masked patches using the embeddings of the visible regions by minimizing the following loss:
\[
\gL_{\text{BEiT}} = -\E_{x_1,x_2}\sum_k \log\Pr(x_{2,k}^+ \mid x_1),
\]
where \( x_1 \) and \( x_2 \) represent the unmasked and masked views of the image, \( x_{2,k}^+ \) denotes the \( k \)-th token of the masked view \( x_2 \).

These reconstruction-based methods, though differing in the specific loss functions used, share the common goal of predicting masked content from incomplete inputs. In the next section, we will reveal their intrinsic connections in the lens of sample alignment.

\section{Masked Pretraining Performs Implicit Contrastive Learning}
\label{sec:mae implicitly aligns positive pairs}

Masked pretraining (MPT) and contrastive learning (CL) are two of the most widely used self-supervised paradigms and both of them exhibit remarkable performance across various downstream tasks \citep{simclr,mae}. While contrastive learning has been extensively studied with several theoretical frameworks analyzing its mechanisms from different perspectives \citep{arora,haochen,infomax}, the theoretical exploration of MPT approaches remains relatively underdeveloped. In this paper, we provide a theoretical analysis of MPT by establishing a fundamental connection between MPT and CL. Specifically, we prove that MPT implicitly aligns semantically similar pairs through its masking mechanism, similar to the alignment between positive pairs in contrastive methods. In Section~\ref{sec:MPT=alignment}, we formalize this relationship, proving that minimizing the pixel-wise MPT reconstruction loss enforces alignment between semantically similar samples. 
Furthermore, in Section~\ref{sec:ce-align-conn}, we extend this analysis to cross-entropy-based MPT methods (e.g., BEiT, BERT), showing that their loss functions implicitly align encoder outputs with token embeddings, mirroring contrastive objectives.

\subsection{Pixel-wise MPT Implicitly Aligns Positive Input Pairs}
\label{sec:MPT=alignment}

Masked pretraining approaches typically involve an encoder-decoder architecture, similar to that of autoencoders. In this context, we start by assuming that the encoder-decoder structure in MPT can perform the fundamental task of a standard autoencoder, \ie reconstructing the original input.

\begin{assumption}
\label{ass:aprroximation pesudo encoder}
For any decoder $g$, let $\mathcal{F}_g\subseteq\mathcal{F}$ denote a non-empty family of pseudo-inverse encoders associated with $g$. We define the best achievable autoencoding error for $g$ as
\[
\varepsilon(g)=\min_{f\in\mathcal{F}_g}\mathbb{E}_{x}\|g\circ f(x)-x\|_2^2,
\]
where $x$ represents either unmasked data $x_1$ or masked data $x_2$. We assume that the minimum is attained by some $f_g\in\mathcal{F}_g$, i.e.,
\[
f_g\in
\arg\min_{f\in\mathcal{F}_g}
\mathbb{E}_{x}\|g\circ f(x)-x\|_2^2.
\]
\end{assumption}

The assumption characterizes the reconstructability of a given decoder $g$ through its associated pseudo-inverse encoders, with $\varepsilon(g)$ quantifying the best achievable autoencoding error. For sufficiently expressive decoders, such reconstructability is plausible in practice, given the strong empirical performance of deep networks on autoencoding tasks \citep{kingma2014auto,jing2020implicit}. Moreover, the Transformer architectures commonly used in MPT possess universal approximation capabilities \citep{Yun2020Are}, further supporting the existence of encoder-decoder pairs with small autoencoding error. We emphasize, however, that $\varepsilon(g)$ is decoder-dependent and need not be uniformly small for all $g$.

However, it is worth noting that the vanilla autoencoder task fails to learn representations as meaningful as those achieved by MPT. For instance, when no masks are applied, the linear probing accuracy of MAE has a significant decline from 61.2\% to 17.4\% on ImageNet-100. This outcome suggests that the autoencoding ability alone, as explored in previous studies \citep{cao2022understand}, is insufficient to explain the effectiveness of MPT. Hence, this motivates us to delve further into the masking mechanism employed by MPT.

\paraemph{MPT Performs Asymmetric Input-Output Alignment on the Mask Graph.} For the sake of simplicity, we assume that MPT employs an $l_2$ reconstruction loss. First, we show that the MPT loss can be lower bounded by an \textit{asymmetric} alignment loss between the two complementary views $x_1,x_2$ using two-branch autoencoders {$h=g \circ f$ (real) and $h_g=g \circ f_g$ (pseudo)}, respectively.
\begin{theorem}
Under Assumption \ref{ass:aprroximation pesudo encoder}, the MPT loss can be lower bounded by
\begin{align}
    \gL_\textnormal{MPT}(h)&\geq\gL_\textnormal{asym}(h)-\varepsilon(g)+\textnormal{const},\\
    \text{and }\gL_\textnormal{asym}(h)&=-\E_{x_1,x_2}h(x_1)^\top h_g(x_2)=-\tr(  H_g ^\top \bar  A_M H), \label{eq:asymmetric-alignment-loss}
\end{align}
\label{thm:asym-align}
where $H$ denotes the output matrix of $h$ on $\gX_1$ whose $x_1$-th row is $H_{x_1}=\sqrt{d_{x_1}}h(x_1)$, and $H_g$ denotes the output matrix of $h_g$ on $\gX_2$ whose $x_2$-th row is $(H_g)_{x_2}=\sqrt{d_{x_2}}h_g(x_2)$. 
\end{theorem}

\begin{wrapfigure}[13]{r}{0.5\textwidth}
    \centering
    \vspace{-6pt}
    \includegraphics[width=\linewidth]{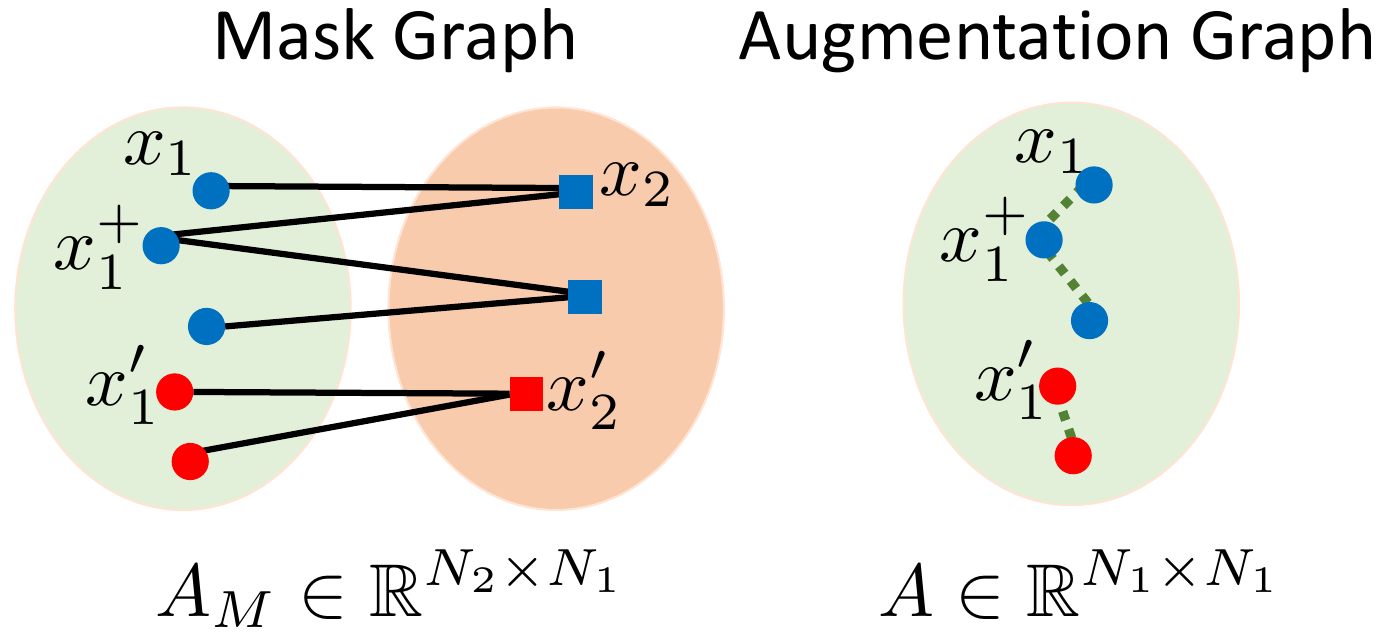}
    \caption{An illustration of the mask graph and the corresponding augmentation graph of MPT. Different colors denote different belonging classes.}
    \label{fig:mask-graph}
    \vspace{-8pt}
\end{wrapfigure}
Consequently, when the MPT loss is small, it implies a reduced alignment loss (as a lower bound). This indicates that MPT, with its encoder-decoder architecture, implicitly aligns the masked and unmasked views. However, it is still unclear why aligning these two complementary views contributes to learning meaningful features. To gain a deeper understanding, we further explore the effect of masking in MPT. We find that via masking, MPT generates \textit{implicit connections} among different \textit{input samples} in the form of \textit{2-hop connectivity}. Consider a pair of 2-hop input neighbors, denoted as $x_1$ and $x_1^+$, both belonging to $\gX_1$, and sharing a common complementary target view, $x_2\in\gX_2$ (more likely to occur under a larger mask ratio $\rho$), as illustrated in Figure \ref{fig:mask-graph}. By enforcing $x_1$ and $x_1^+$ to reconstruct the same output $x_2$, MPT implicitly maps their features together. In this way, the 2-hop input neighbors serve as \textit{positive pairs} that are implicitly aligned as in contrastive learning.

Motivated by this observation, we seek to establish a formal connection between MPT and contrastive learning below. Specifically, we will show that similar to the data augmentations employed in contrastive learning, the masking mechanism in MPT also introduces implicit affinities between \textit{input samples} and creates (implicit) \textit{positive pairs}. To present this connection formally, we introduce an augmentation graph $\gG_A$, which models the relationship between all input samples in $\gX_1$. It is important to note that this augmentation graph $\gG_A$ is distinct from the mask graph $\gG_M$, which captures the input-output relationship between $\gX_1$ and $\gX_2$.

\paraemph{The Augmentation Graph $\gG_A$ of MPT.} 
To model the mask-induced affinity between all unmasked views in $\gX_1$, we construct an augmentation graph $\gG_A$.\footnote{We can also construct an augmentation graph $\gG'_A$ on the output space $\gX_2$, where the $(x_2,x_2')$-th element of the adjacency matrix $A'$ is defined by $\gA(x_2,x_2')=\E_{x_1}[\gM(x_2|x_1)\gM(x_2'|x_1)]$. As the results are quite similar, we mainly take the input space as an example in this work.} Specifically, for any pair of views $x_1$ and $x'_1$ in $\gX_1$, we define the edge weight in the adjacency matrix $A$ as the probability of having the same target view, \ie $\gA(x_1,x_1') = \E_{x_2}[\gM(x_1|x_2)\gM(x'_1|x_2)]$.\footnote{$\gM(x_1|x_2)=\gM(x_1,x_2)/\gM(x_2)$ can further be calculated by marginalizing $\gM(x_1,x_2|\bar x)$.} With this formulation, we define a \textit{symmetric} alignment loss with respect to the positive pairs $(x_1,x_1^+)\sim\gA(x_1,x_1^+)$ drawn according to the augmentation graph:
\begin{equation}
 \gL_\text{align}(h)=-\E_{x_1,x_1^+}h(x_1)^\top h(x_1^+).
 \label{eq:alignment-loss}
\end{equation}

\paraemph{MPT Performs Symmetric Input Alignment on the Augmentation Graph.} 
Built upon the augmentation graph constructed above, we theoretically verify the intuition on 2-hop connectivity. To achieve this, we establish a relationship between the asymmetric input-output alignment loss on the mask graph and the symmetric input alignment loss on the augmentation graph in the following theorem.
\begin{theorem}
The asymmetric alignment loss on the mask graph (Eq.~\ref{eq:asymmetric-alignment-loss}) can be lower bounded by the symmetric alignment loss on the augmentation graph (Eq.~\ref{eq:alignment-loss}):
\begin{equation}
\gL_\textnormal{asym}(h)\geq \frac{1}{2}\gL_\textnormal{align}(h)+\textnormal{const}.
\end{equation}
\paraemph{Proof Sketch.} We provide a proof sketch of this inequality as it is the key to our analysis. We first symmetrize the asymmetric alignment loss $\gL_\textnormal{asym}(h)$ with an arithmetic inequality, and then establish its equivalence to the symmetric alignment loss $\gL_\textnormal{align}(h)$. The derivation highlights the intrinsic connection between the mask graph (adjacency matrix $A_M$) and the augmentation graph (adjacency matrix $A$).
\begin{equation*}
    \begin{aligned}
    \gL_\textnormal{asym}(h)&=\E_{x_1,x_2}h(x_1)^\top h_g(x_2)\\
    &=-{\color{blue}\tr( H_g^\top \bar A_M H )} & (\text{reformulated to the mask graph (Theorem \ref{thm:asym-align}}))\\
    &\geq-\frac{1}{2}(\| \bar A_M H\|^2+\|H_g\|^2) &(\text{because }\tr(AB)\leq\frac{1}{2}(\|A\|^2+\|B\|^2)) \\
    &=-\frac{1}{2}{\color{blue}\tr(H^\top \bar A_M^\top \bar A_M H)}-\frac{1}{2} &(\text{because }\|H_g\|^2=\sum_{x_2}d_{x_2}\|h_g(x_2)\|^2=1)\\
    &=-\frac{1}{2}{\color{blue}\sum_{x_1,x'_1}A_{x_1,x'_1}h(x_1)^\top h(x'_1)}-\frac{1}{2} &\text{(transformed to the augmentation graph)}\\
    &=\frac{1}{2}\gL_\textnormal{align}(h)-\frac{1}{2}. &\text{(following the definition in Eq.~\ref{eq:alignment-loss})}
    \end{aligned}
\end{equation*}
\label{thm:asym-sym}
\end{theorem}
Combining Theorem \ref{thm:asym-align} and Theorem \ref{thm:asym-sym}, we derive the main theorem, which shows that MPT's reconstruction loss can be bounded by the symmetric alignment loss of the positive input pairs $(x_1,x_1^+)$ drawn according to the augmentation graph. We provide visualization examples of the augmentation graph in Appendix \ref{sec:vis of augmentation graph}.
\begin{theorem}
Under Assumption \ref{ass:aprroximation pesudo encoder}, MPT's reconstruction loss (Eq.~\ref{eq:mae-loss}) can be lower bounded by the alignment loss between positive pairs $(x_1,x^+_1)\sim\gA(x_1,x_1^+)$,
\begin{equation}
    \gL_\textnormal{MPT}(h) \geq \frac{1}{2}\gL_\textnormal{align}(h)-\varepsilon(g)+ \textnormal{const} = -\frac{1}{2} \E_{x_1,x_1^+}h(x_1)^\top h(x_1^+) -\varepsilon(g)+ \textnormal{const}.
\end{equation}
\label{thm:implicit-alignment}
\end{theorem}
\vspace{-0.2in}
In this way, we establish a close relationship between the two prominent SSL paradigms (MPT and contrastive learning) by showing that a small MPT loss will imply a small alignment loss of positive input pairs as in contrastive learning.
Leveraging this connection, we could establish guarantees for the downstream generalization of MPT (discussed in Section \ref{sec:downstream}). 

\paraemph{Comparison to Contrastive Learning.} Comparing the alignment loss (Eq.~\ref{eq:alignment-loss}) to that of contrastive learning \citep{simclr,haochen}, we observe that the main difference lies in that contrastive learning aligns features in the \textit{latent} space of the encoder $f$, while MPT aligns features in the \textit{output} space with an encoder-decoder architecture $h=g\circ f$. Nevertheless, it is worth noting that most variants of contrastive learning apply a nonlinear projection head $g$ following the encoder $f$ before calculating the alignment loss \citep{simclr,simclrv2,moco,BYOL,simsiam,ouyang2025projection}. Similarly, MPT also utilizes a decoder $g$. Therefore, we may also regard the role of MPT's decoder as the projection head in contrastive learning.
Theoretically, if we further assume the bi-Lipschitzness of the decoder, we can show that the MPT loss is further lower bounded by an alignment loss defined in the feature space. 
\begin{corollary}
Under Assumption \ref{ass:aprroximation pesudo encoder} and the assumption that the decoder is $L$-bi-Lipschitz, \ie $\forall\ (x_1,x_2), \frac{1}{L}\| x_1 -x_2\|^2\leq\|g(x_1)-g(x_2)\|^2 \leq L \| x_1 -x_2\|^2.$
Then, the MPT reconstruction loss can be lower bounded by the alignment loss \wrt the encoder outputs:
\begin{equation}
\gL_\textnormal{MPT}(h)\geq -\frac{1}{2L}\cdot \E_{x_1,x_1^+} f(x_1) ^\top f(x_1^+) -\varepsilon(g) + \textnormal{const}.
\label{eq:MPT-alignment}
\end{equation}
\label{cor:implicit-alignment}
\end{corollary}

\subsection{Cross-Entropy Loss Is Implicitly Equal to Token-Level Alignment}
\label{sec:ce-align-conn}

While pixel-wise reconstruction-based MPT methods align features through masked prediction, we further analyze whether cross-entropy-driven approaches like BEiT and BERT also achieve alignment indirectly via token prediction in this section.

Incorporating Eq.~(\ref{eqn:ce-softmax}), the cross-entropy (CE) loss, which is useful in models such as BERT and BEiT, can be uniformly reformulated as:

\begin{align}
\gL_{\textnormal{CE}} &= -\E_{x_1,x_2} \sum_k\log \Pr(x_{2,k}^+|x_1) \nonumber \\
&= \sum_{k}\left[-\E_{x_1,x_2} f(x_1)^\top w_{x_{2,k}^+} + \E_{x_1,x_2}\log \sum_{x_{2,k}^-} \exp(f(x_1)^\top w_{x_{2,k}^-})\right],
\end{align}
where $x_{2,k}^+$ is a ground-truth token and $x_{2,k}^-$ is an irrelevant token. The core objective of CE loss is to maximize the predicted probability of corresponding masked tokens while minimizing the probability of irrelevant tokens. Following the same spirit, we adopt the spectral loss \citep{haochen} for simplicity of analysis, i.e.:
\begin{align}
\gL_{\textnormal{spectral}} = -\sum_{k}\E_{x_1,x_2}f(x_1)^\top w_{x_{2,k}^+} + \sum_{k}\E_{x_1,x_2,x_{2,k}^-}\left[f(x_1)^\top w_{x_{2,k}}^-\right]^2.
\end{align}
From this formulation, we observe that the first term of the cross-entropy loss encourages alignment between the predicted features of the unmasked portion ($f(x_1)$) and the feature vectors of the masked tokens in the classification head $(w_{x_{2,k}^+})$. Consequently, when two unmasked views are aligned with the same masked view, they are also implicitly aligned with each other. As a result, we know that MPT methods with CE loss also implicitly align the unmasked views that share the same masked views.

\section{Avoiding Dimensional Collapse in MPT with an Explicit Uniformity Regularization}

In Section \ref{sec:mae implicitly aligns positive pairs}, we have demonstrated that the MPT loss is implicitly equal to an alignment loss. However, in contrastive learning, simply aligning the positive pairs leads to the problem of full feature collapse, because the alignment loss can also be minimized when the encoder produces a constant feature for all inputs.  Interestingly, MPT manages to avoid full feature collapse. This raises the question of how MPT achieves this property and we discuss it in this section. Specifically, in Section \ref{sec:implicit-uniformity}, we reveal that the MPT loss can avoid full feature collapse while still suffering from dimensional collapse. In Section \ref{sec:u-MPT}, we propose uniformity-enhanced MPT (U-MPT) loss to solve the dimensional collapse issue. Empirically, we verify the effectiveness of the U-MPT objective in Section \ref{U-MPT on real-w datasets.}. Furthermore, in Section \ref{sec:empirical understandings of u-mpt}, we empirically analyze the advantages of U-MPT and the choices of hyperparameters.

\subsection{The Feature Collapse Issue in MPT}
\label{sec:implicit-uniformity}

As shown in Section \ref{sec:mae implicitly aligns positive pairs}, during the pretraining process, MPT implicitly aligns the semantically similar images. However, we note that simply aligning the positive pairs may lead to the full feature collapse, i.e.,  the network can minimize the alignment loss by encoding all input into a constant. To address this issue, various techniques have been proposed in contrastive learning, such as incorporating additional losses to encourage feature uniformity or decorrelation \citep{simclr, moco, barlowtwins}, or employing asymmetric structural designs \citep{BYOL, simsiam, swav}. Recent theoretical work further attributes the collapse-avoidance effect of such asymmetric designs to a rank differential mechanism that improves the effective dimensionality of learned representations \citep{zhuo2023towards}. 
However, MPT based on the pixel-wise reconstruction loss avoids this issue without special designs. In the following theorem, we show that minimizing the MPT loss can provably get rid of the full feature collapse.
\begin{theorem}
When the encoder fully collapses, \ie $\forall x\in\gX_1,f(x)=c$, the MPT loss has a large lower bound:
\begin{equation}
    \gL_\textnormal{MPT}(h)\geq \Var(x_2),
\end{equation}
where $\Var(x_2)$ denotes the variance of masked targets computed on the training data set.
\label{thm:MPT not tc}
\end{theorem}
\paraemph{MPT Can Avoid Full Feature Collapse.}
Since the training data contain diverse images, the variance $\Var(x_2)$ is relatively large. Consequently, unlike the alignment loss in contrastive learning, the MPT loss cannot be minimized (to a small value) by a fully collapsed encoder. The primary reason is that the alignment loss operates in a fully flexible latent space that allows a collapsed encoder to minimize the loss, while in MPT, the reconstruction loss adopts a parameter-invariant and sample-dependent target $x_2$. As a result, a fully collapsed encoder is unable to minimize the reconstruction loss with respect to $x_2$. Consequently, MPT is inherently resistant to full feature collapse.

\paraemph{MPT Still Suffers from Dimensional Collapse.} Although MPT can prevent full feature collapse, it may still suffer from \emph{dimensional feature collapse} where the learned features lie in a low-dimensional subspace \citep{hua2021feature,jing2021understanding}, which also limits its representation power.
To illustrate that this issue can arise even under optimal reconstruction, we construct a simple example in which perfect masked reconstruction is compatible with a dimensionally collapsed representation.
\begin{prop}
Consider a masked reconstruction problem with a data set of $N$ images, where each image consists of $K$ identical patches:
\[
x_i=[p_i,p_i,\cdots,p_i],
\]
where $p_i\in\mathbb{R}^{s}$ and $p_i\neq p_j$ for any $i\neq j$. Let the representation space be $\mathbb{R}^{d}$ with $d\geq 2$. For any mask operator $m$ that leaves at least one patch visible, there exists an encoder-decoder pair $(f,g)$ that achieves perfect masked reconstruction while all learned representations lie in a one-dimensional subspace of $\mathbb{R}^{d}$.
\end{prop}

To see this, let $v\in\mathbb{R}^{d}$ be a fixed non-zero vector and let
$Q_1,\ldots,Q_N$ be distinct scalars. Since any non-empty visible view
$x_i[m]$ contains the patch $p_i$, and $p_i\neq p_j$ for $i\neq j$, the visible view
uniquely identifies the original image. We can therefore construct an encoder
such that
\[
f(x_i[m])=Q_i v.
\]
A sufficiently expressive decoder can then map each distinct code $Q_i v$
back to the corresponding image:
\[
g(f(x_i[m]))=x_i,\qquad i=1,\ldots,N.
\]
Thus, the reconstruction loss achieves its minimum of zero. However, all representations lie in the one-dimensional subspace
$\mathrm{span}(v)$. Specifically, the representation matrix satisfies
\[
Z=
\begin{bmatrix}
Q_1\\
\vdots\\
Q_N
\end{bmatrix}
v^\top,
\]
and hence $\mathrm{rank}(Z)=1$, although the representation space itself is $d$-dimensional. This counterexample therefore formally demonstrates that optimal masked reconstruction does not, by itself, prevent dimensional collapse.

We further investigate whether such dimensional feature collapse occurs in real-world data by conducting experiments on ImageNet-100. Figure \ref{fig:singular-value} shows that after learning, MPT's features exhibit a higher degree of collapse, as indicated by the reduced number of large singular values (representing non-collapse dimensions). Quantitatively, Figure \ref{fig:effective-rank} shows that during the training process, the features of MPT progressively collapse, resulting in a smaller effective rank \citep{roy2007effective}. This trend confirms that MPT experiences an increasing level of dimensional feature collapse.

\begin{figure}
    \centering
    \subfigure[singular values]{
    \includegraphics[width=.4\textwidth]{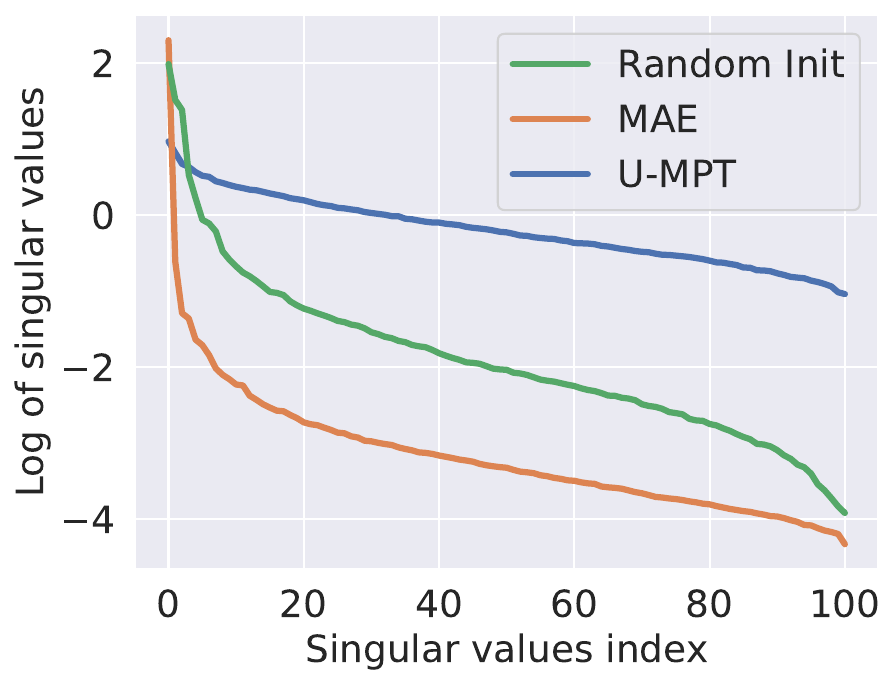}
    \label{fig:singular-value}
    }
    \subfigure[effective rank]{
    \includegraphics[width=.4\textwidth]{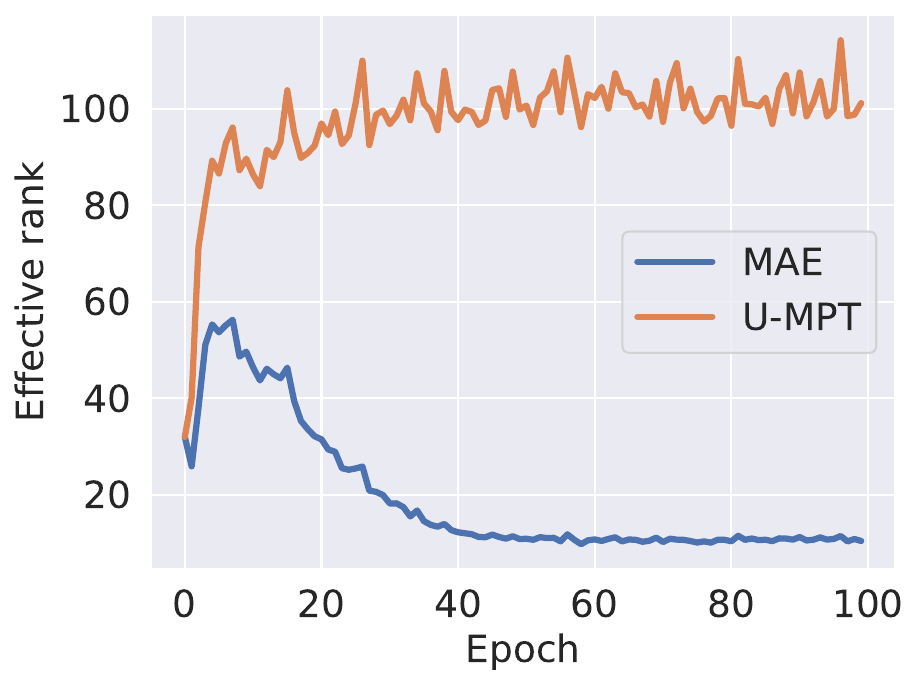}
    \label{fig:effective-rank}
    }    
    \caption{
    (a) Comparison of the singular values of learned features with 1) random initialization, 2) MPT loss, and 3) our U-MPT loss.\\
    (b) The changing process of effective rank \citep{roy2007effective} of the encoded features trained with different objectives (MPT and U-MPT).}
\end{figure}

\subsection{U-MPT: Uniformity-enhanced Masked Pretraining Objective}

\label{sec:u-MPT}
To further enhance the feature diversity of MPT and alleviate the dimensional collapse issue, we draw inspiration from the uniformity loss in contrastive learning \citep{simclr, haochen} and propose the Uniformity-enhanced MPT (U-MPT) loss, which incorporates an explicit regularization term to promote feature uniformity through a coefficient $\lambda>0$,
\begin{align}
    \gL_\textnormal{U-MPT}(h)&=\gL_\textnormal{MPT}(h)+\lambda\cdot \gL_\textnormal{unif}(f),\label{eq:u-MPT}
    \\
    \textnormal{where }\gL_\textnormal{unif}(f)&=\E_{ x_1}\E_{x^-_1}(f(x_1)^\top f(x_1^-))^2,
\end{align}
and ${x}^-_1$ denotes an independently drawn unmasked view from $\gX_1$. Intuitively, the spectral uniformity loss $\gL_\text{unif}(f)$ \citep{haochen} encourages a small feature similarity between random unmasked views, which could effectively promote the feature diversity of all samples. As a preview of the results, we can observe from Figures \ref{fig:singular-value} \& \ref{fig:effective-rank} that U-MPT effectively addresses the dimensional feature collapse issue and significantly improves the effective feature dimensionality by a large margin.

Theoretically, combined with Corollary \ref{cor:implicit-alignment}, we can show that the U-MPT loss is lower-bounded by the spectral contrastive loss with a specific choice of $\lambda$.
\begin{theorem}
Denote the spectral contrastive loss (SCL) from  \citet{haochen} as
\begin{equation}
    \gL_\textnormal{SCL}(f)=2\gL_\textnormal{align}(f)+\gL_\textnormal{unif}(f)=-2\E_{x_1,x_1^+}f(x_1)^\top f(x_1^+)+\E_{x_1,x_1^-}(f(x_1)^\top f(x_1^-))^2.
    \label{eq:scl}
\end{equation}
Under Assumption \ref{ass:aprroximation pesudo encoder} and the assumption that the decoder is $L$-bi-Lipschitz, when $\lambda=1/(4L)$, the U-MPT loss can be lower bounded by the SCL loss:
\begin{equation}
\gL_\textnormal{U-MPT}(h)\geq\frac{1}{4L}\cdot\gL_\textnormal{SCL}(f)-\varepsilon(g) + \frac{1}{2}.
\label{eq:U-MPT-lowerbound}
\end{equation}
\label{thm:spectral-loss}
\vspace{-0.2in}
\end{theorem}
As a result, minimizing the U-MPT loss will implicitly minimize the spectral contrastive loss among input views. This reveals that, 
by adding a uniformity regularization term, U-MPT aligns well with contrastive learning. Next, we will present the main empirical results of our proposed U-MPT loss on different real-world data sets with different backbones, and then conduct a series of experiments to understand the advantages of U-MPT loss.

\subsection{Evaluation on Benchmark Data Sets}
\label{U-MPT on real-w datasets.}
To evaluate the effectiveness of the proposed U-MPT loss, extensive experiments are conducted on CIFAR-10 \citep{cifar}, ImageNet-100 \citep{imagenet}, and ImageNet-1K \citep{imagenet} with MAE~\citep{mae} being the instantiated MPT paradigm.

\paraemph{Setup.}
We mainly follow the basic setup of MAE: for the encoder, we adopt different variants of ViT \citep{dosovitskiy2020image}, \ie ViT-Tiny, ViT-Base, and ViT-Large. For the decoder, we use a flexible one following \cite{mae}. The mask ratio is set to 0.75. The coefficient of the uniformity term in the U-MPT loss is set to 0.01. On CIFAR-10, we pretrain the model for 2000 epochs with batch size 4096 and weight decay 0.05. On ImageNet-100 and ImageNet-1K, we pretrain the model for 200 epochs with batch size 1024 and weight decay 0.05.

\paraemph{In-domain Linear Probing and Fine-tuning.} We conduct both linear probing and end-to-end fine-tuning on the pretrained encoder, where pretraining and downstream evaluation use the same data set. For linear probing, the pretrained encoder is frozen and only a linear classifier is optimized, which directly measures the quality of the pretrained representation. As for end-to-end fine-tuning, both the pretrained encoder and the classification head are jointly optimized, which evaluates whether the advantage induced during pretraining persists after full supervised adaptation on the same data distribution. In Table \ref{table:fintune results}, we compare the performance of original MAE loss and U-MPT loss on different benchmarks. We find that, on linear evaluation results, our proposed U-MPT loss average increases 8.9\% on CIFAR-10, 7.35\% on ImageNet-100, and 3.35\% on ImageNet-1K with two different backbones. On fine-tuning results, our proposed U-MPT loss will not hurt the performance of fine-tuning results of MAE. This suggests that full downstream adaptation can reduce the performance gap induced by different pretraining objectives, whereas the advantage of U-MPT is more directly reflected in the quality of the pretrained representations themselves.

\begin{table}[!htbp]
\centering
\footnotesize
\begin{tabular}{llcccccc}
\toprule
                                           &  & \multicolumn{2}{c}{CIFAR-10} & \multicolumn{2}{c}{ImageNet-100} & \multicolumn{2}{c}{ImageNet-1K} \\ 
Downstream Task                                  & loss          & ViT-Tiny      & ViT-Base     & ViT-Base        & ViT-Large      & ViT-Base       & ViT-Large      \\ \midrule
\multirow{2}{*}{Linear Probing} & MAE       &  59.6          &61.7                 & 61.2            & 64.4           & 55.4           &62.2                \\
                                           & U-MPT     & \textbf{68.9}                 & \textbf{70.2}           & \textbf{67.5}   & \textbf{72.8}  & \textbf{58.5}           &    \textbf{65.8}            \\ \midrule
\multirow{2}{*}{Fine-tuning}          & MAE       & \textbf{89.6}          & 90.7         & \textbf{86.9}            & \textbf{87.3}           &    82.9            &      \textbf{83.3}          \\
                                           & U-MPT     & 89.4          & \textbf{90.8}         & {86.8}            & \textbf{87.3}           &       \textbf{83.0}         &      83.2          \\ \bottomrule
\end{tabular}
\caption{In-domain linear-probing accuracy (\%) and fine-tuning accuracy (\%) of models pretrained by MAE loss and U-MPT loss with different ViT backbones on CIFAR-10, ImageNet-100, and ImageNet-1K. The uniformity {regularizer} term in the U-MPT loss significantly improves the linear evaluation performance of the MAE loss without hurting the performance of fine-tuning.}
\label{table:fintune results}
\end{table}

\paraemph{Cross-dataset Fine-tuning.} 
We next evaluate whether the benefit of U-MPT persists when the pretrained encoder is transferred to and fully adapted on a downstream data set different from the pretraining data set. Starting from models pretrained on ImageNet, we perform end-to-end fine-tuning on four downstream benchmarks with diverse characteristics: 1) iNaturalist-2021 \citep{inat}, a large-scale and long-tailed fine-grained species classification data set; 2) CUB \citep{wah2011caltech}, a fine-grained bird classification benchmark; 3) Cars \citep{car197}, a fine-grained car model classification benchmark; and 4) ImageNet-A \citep{ImageNet-A}, a challenging data set of naturally occurring images that exhibits a substantial distribution shift from the standard ImageNet distribution. As shown in Table~\ref{table:fine-grained}, U-MPT improves performance across all four downstream data sets. The improvement is particularly pronounced on iNaturalist-2021, where accuracy increases from 66.0\% to 69.4\%. On ImageNet-A, U-MPT also improves accuracy from 18.4\% to 19.9\%. These results suggest that the benefit of U-MPT can remain visible even after the pretrained encoder is fully adapted to a shifted downstream data set.

\begin{table}[!h]
\vspace{0.1cm}
\renewcommand{\arraystretch}{1.2}
\centering
\setlength{\tabcolsep}{6.0mm}{
\begin{tabular}{lcccc}
\hline
 Method    &  iNat$_{21}$   &  CUB & Cars & ImageNet-A   \\ \hline
MAE  & 66.0  & 81.5 & 59.8 & 18.4 \\
\rowcolor{gray!20}  U-MPT  &  \textbf{69.4}  & \textbf{81.8} & \textbf{60.3} & \textbf{19.9}\\ \hline
\end{tabular}}
\caption{Cross-dataset fine-tuning accuracies (\%) of models pretrained on ImageNet using MAE and U-MPT objectives. The pretrained models are subsequently fine-tuned and evaluated on different downstream data sets.}
\label{table:fine-grained}
\end{table}

\begin{table}[!h]
\renewcommand{\arraystretch}{1.2} %
\centering
\setlength{\tabcolsep}{6.0mm}{ 
\begin{tabular}{lcccc}
\hline
Data Set and Corruption Type & MAE & U-MPT    \\ \hline
ImageNet-A & 0.57 & \textbf{0.68} \\
ImageNet-C (Defocus Blur) & 10.48 & \textbf{26.20} \\
ImageNet-C (Motion Blur) & 16.36 & \textbf{35.84} \\
ImageNet-C (Frosted Glass Blur) & 13.68 & \textbf{30.14} \\
ImageNet-C (Zoom Blur) & 14.24 & \textbf{32.48} \\
ImageNet-C (Elastic Transform) & 23.02 & \textbf{39.98} \\
ImageNet-C (Contrast) & 9.76 & \textbf{23.40} \\
ImageNet-C (Brightness) & 28.52 & \textbf{50.14} \\
ImageNet-C (Pixelate) & 24.66 & \textbf{46.18} \\
ImageNet-C (Snow) & 9.86 & \textbf{22.84} \\
ImageNet-C (Fog) & 10.52 & \textbf{24.62} \\
ImageNet-C (Frost) & 11.30 & \textbf{23.64} \\
ImageNet-C (JPEG Compression) & 20.36 & \textbf{39.24} \\
ImageNet-C (Gaussian Noise) & 11.46 & \textbf{24.94} \\
ImageNet-C (Impulse Noise) & 8.78 & \textbf{20.08} \\
ImageNet-C (Shot Noise) & 11.40 & \textbf{24.08} \\
\hline
\end{tabular}}
\caption{Out-of-distribution generalization accuracies (\%) of MAE and U-MPT on ImageNet-A and ImageNet-C. The pretrained encoder is kept frozen and evaluated on ImageNet-A and ImageNet-C using linear probing.} 
\vspace{0.1cm}
\label{table:ood}
\end{table}

\paraemph{Out-of-Distribution Generalization.} We further evaluate the quality of the frozen ImageNet pretrained representations under distribution shift using
ImageNet-A \citep{ImageNet-A} and ImageNet-C \citep{ImageNet-C}. 
Specifically, models are pretrained on ImageNet, after which the pretrained encoder is kept frozen and evaluated on ImageNet-A and ImageNet-C using linear probing. In contrast to the cross-dataset fine-tuning setting above, the encoder is not adapted to the target distribution, allowing us to more directly assess the generalization ability of the pretrained representations under distribution shift. 
The results presented in Table~\ref{table:ood} demonstrate that incorporating the uniformity loss significantly improves the classification accuracies on ImageNet-A and each corruption type of ImageNet-C. We believe that this improvement is due to the uniformity loss encouraging greater feature diversity, which helps to separate the feature representations of different classes. Consequently, when corruptions are applied, the features are less likely to be misclassified into the feature space of another class. Similar connections between improved feature separation and robustness under distribution shifts have also been observed in prior work \citep{zhang2025beyond}.

\paraemph{Extension to Other MIM Frameworks.} 
We also verify our uniformity-enhanced loss in another MIM framework SimMIM \citep{xie2022simmim}. For SimMIM, we use ViT-Base as the encoder and use the linear decoder as in \cite{xie2022simmim}. We use the recommended mask ratio 0.6. The coefficient of the uniformity term is set to 0.01, the same as in MAE. For ImageNet-100, we pretrain the model for 200 epochs with batch size 128 and weight decay 0.05. Linear evaluation is conducted in Table \ref{table: offline SImMIM}. We can see that the linear accuracy increases by 6.8\% with the uniformity regularizer, which further verifies the general property of our approach.

\vspace{10pt}
\begin{minipage}[!t]{\textwidth}
\begin{minipage}[!t]{0.43\textwidth}
\centering
\makeatletter\def\@captype{table}
\begin{tabular}{cc}
\toprule
SimMIM & With Uniformity Loss       \\ \midrule
54.3 & \textbf{61.1} \\ \bottomrule
\end{tabular}
\caption{Linear probing accuracy (\%) of uniformity-enhanced SimMIM on ImageNet-100.}
\label{table: offline SImMIM}
\end{minipage}
\space\space\space\space\space\space\space\space
\begin{minipage}[!t]{0.43\textwidth}
\centering
\makeatletter\def\@captype{table}
\begin{tabular}{lccccc}
\toprule
$\lambda$       & 0   &1e-3 & 1e-2 &  1e-1  &1e0\\ \midrule
Acc & 61.2 &65.9  & \textbf{67.5}   & 45.1  &47.0    \\ \bottomrule
\end{tabular}
\caption{Linear probing accuracy (\%) of uniformity-enhanced MAE with different coefficients.}
\label{tab:different lambda}
\end{minipage}
\end{minipage}

\subsection{Empirical Understandings}
\label{sec:empirical understandings of u-mpt}
To further understand the empirical behavior of U-MPT, we analyze the effect of the regularization coefficient, the geometry of the learned representations, and the training dynamics.

\paraemph{Different Coefficients of the {Regularizer} Term.} The most important hyper-parameter of our proposed U-MPT loss is the coefficient of the uniformity {regularizer} term. In Table \ref{tab:different lambda}, we present the results of linear evaluation {on ImageNet-100} trained with the U-MPT loss with different coefficients of the {regularizer} term. We can see that the downstream performance increases when the coefficient increases from 0 to 0.01. However, the overlarge coefficient will also hurt the performance of U-MPT loss as the task of MAE will be overlooked.

\paraemph{Visualization of Representations.} To intuitively understand the improvement of our U-MPT loss on clustering intra-class samples, we use t-SNE \citep{vandermaaten08a} to visualize the representations trained with MAE loss and U-MPT loss on ten random classes of ImageNet-100. As shown in Figure \ref{fig:tsne}, we find that with our uniformity {regularizer} term, the samples are much better-clustered corresponding to their ground-truth labels. {To be specific, the red class (``hens'') and the gray class (``indigo birds'') are separated from others, this is because most of other classes are the animals living in the oceans while these two classes are more like the birds living on the land or the sky. Thus, these two classes are easier to distinguish, especially with our uniformity regularizer.}

\begin{figure}[!htbp]
    \centering
    \subfigure[t-SNE representations]{
    \includegraphics[width=.6\textwidth]{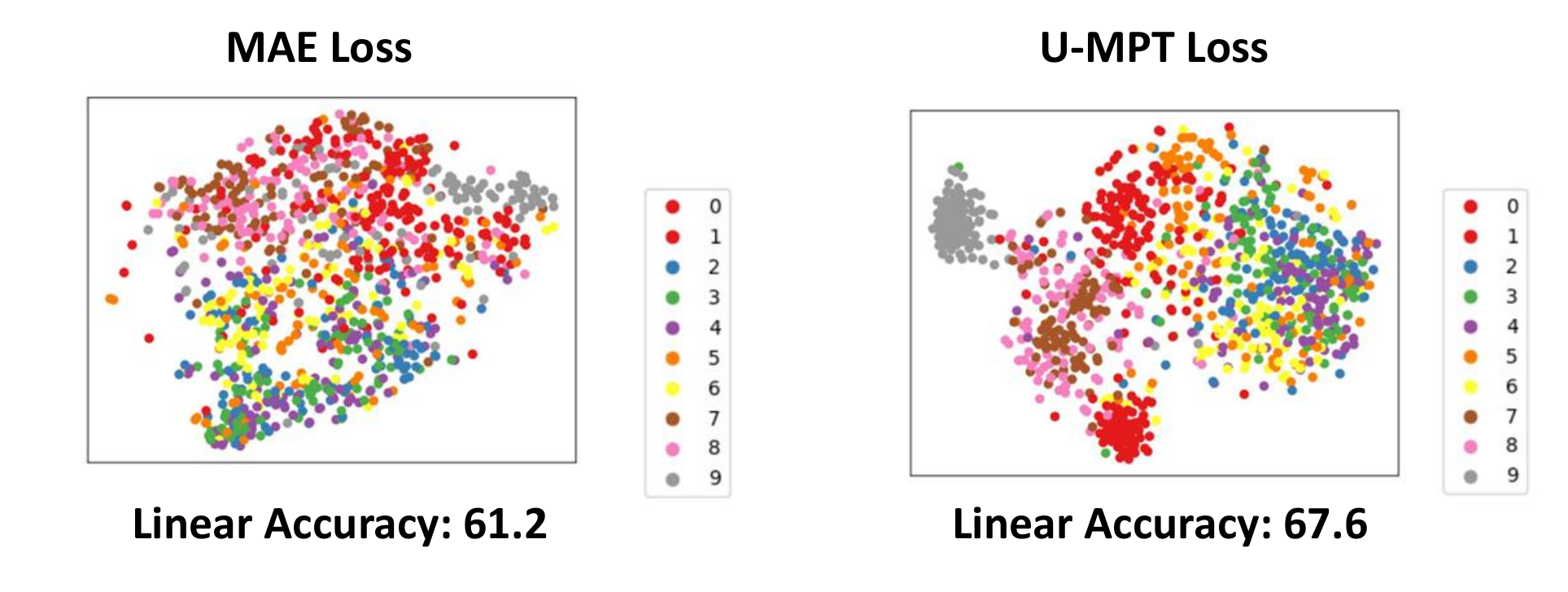}
    \label{fig:tsne}
    }
    \subfigure[Comparison between original MAE loss and our U-MPT loss along training.]{
    \includegraphics[width=0.35\textwidth]{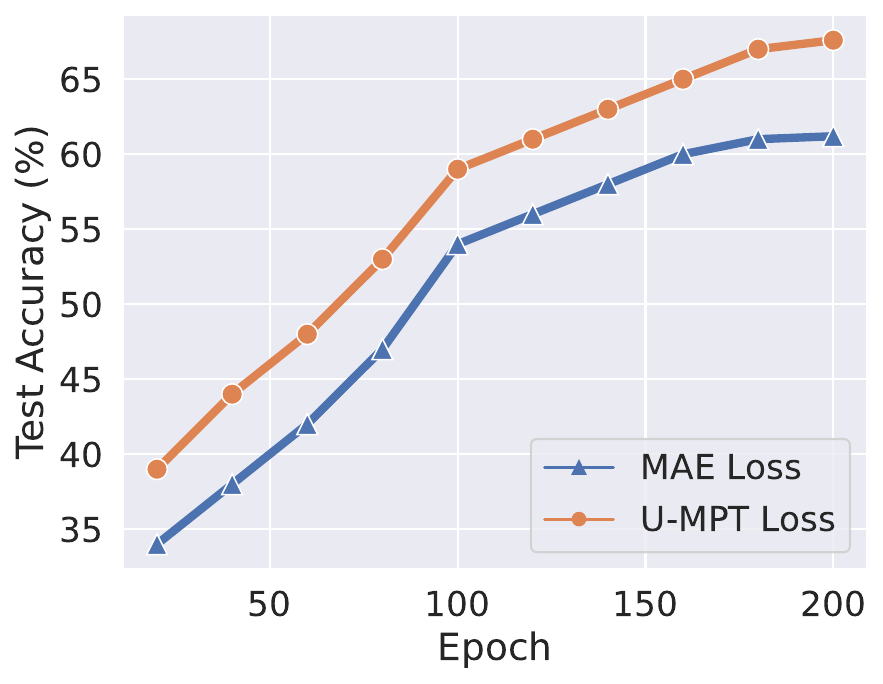}
    \label{fig:different training time}
    }
    \caption{(a) Visualization of representations on random 10 classes of ImageNet-100 trained with MAE loss and our U-MPT loss. Our U-MPT loss significantly improves the class-clustering performance of the encoder. 
    (b) The linear evaluation results during the training process. It shows that our U-MPT loss improves the downstream performance consistently along training. }
\end{figure}

\paraemph{Training Curve.} To further compare the performance between the original MAE loss and U-MPT loss, we plot the linear evaluation accuracy {on ImageNet-100} during the training process in Figure \ref{fig:different training time}. We can observe that our proposed U-MPT loss improves the performance of MAE with all different training epochs, which verifies the effectiveness of our proposed U-MPT loss.

\section{Further Theoretical Generalization Analysis}
\label{sec:downstream}

Having established that MPT implicitly aligns positive pairs through its masking mechanism, we now turn to a theoretical analysis of its downstream generalization performance based on this connection. Without loss of generality, we take the pixel-wise U-MPT objective as an example. In Section \ref{sec:downstream-insight}, we derive the first theoretical guarantees linking MPT’s pretraining loss to downstream classification error. In Section \ref{sec:mask-empirical}, we reveal how the mask ratio $\rho$ influences the downstream generalization error through a toy model. Furthermore, based on the theoretical analysis, in Section \ref{sec:metric}, we propose a surrogate metric to estimate the downstream error and find that it correlates well with the empirical designs of MPT.

\subsection{Theoretical Guarantees on Downstream Classification}
\label{sec:downstream-insight}

In this part, we analyze the downstream performance of U-MPT on the $c$-class linear classification task \citep{arora,haochen}. The performance is measured by the prediction accuracy of natural images $\bar x$ \wrt their labels $y({\bar x})$ when applying a linear prediction head upon pretrained features. For simplicity, we adopt the mean classifier $p_f(x)=\argmax W_ff(x)$, where $W_f\in\sR^{c\times k} $ is the weight of the linear classification head, and for $y\in[c]$, the $y$-th row $W_y=\E_{x_1|y}f(x_1)^\top$ contains the mean representation of the class $y$. It has been empirically demonstrated by \citet{arora} that the mean classifier achieves comparable performance to learnable linear heads. And we define an ensembled linear predictor $p_f'$. For an original sample $\bar{x}$, the predictor ensembles the predictions from all different views and selects the label with the highest frequency. By default, we set the uniformity coefficient $\lambda=1/(4L)$ as in Theorem \ref{thm:spectral-loss}.
\begin{theorem}
Denote the mask-induced label error as $\alpha=\E_{\bar{x},x_1}\mathbbm{1}[{y(x_1)\neq y(\bar{x})}]$
. Then, for $\forall\ h\in\gH$ (the hypothesis class) with $h=g\circ f$, the downstream classification error of its encoder can be upper bounded by its U-MPT pretraining loss:
\begin{align}
\operatorname{Pr}\left(y(\bar x)\neq p_f'(\bar x)\right)\leq c_1L\cdot\gL_\textnormal{U-MPT}(h)+c_2\alpha+c_3L \varepsilon(g)+c_4
,
\end{align}
where $c_1,\dots,c_4$ are constants and $c_3>1$.
\label{thm:downstream and MPT}
\end{theorem}
This theorem establishes an upper bound on the downstream error of an encoder $f$ trained with the U-MPT pretraining loss, which is the \emph{first theoretical guarantee} on the downstream performance of MPT methods. As an implication of this theorem, a small U-MPT loss would provably imply a small downstream classification error, which helps explain the strong downstream generalization ability of MAE \citep{mae}. 
Additionally, we establish a common lower bound on the U-MPT loss that holds for all $h\in\gH$. As a large common lower bound of U-MPT loss indicates that the downstream error will always have a large upper bound, we should pursue a small common lower bound in the following theorem.
\begin{theorem}
The U-MPT pretraining loss has the following common lower bound:
\begin{align}
\forall\ h\in\gH,\quad \gL_\textnormal{U-MPT}(h)&\geq\frac{1}{4L}{\sum_{i=k+1}^{N_1}\lambda_i^2}-\varepsilon(g)+\textnormal{const},
\end{align}
where $\lambda_1\geq\cdots\geq\lambda_{N_1}$ denote the eigenvalues of $A$.
\label{thm:downstream-spectrum}
\end{theorem}

Combining Theorem \ref{thm:downstream and MPT} and
Theorem \ref{thm:downstream-spectrum}, it is evident that the downstream error of MPT can be minimized with three factors: a small vanilla autoencoding error $\varepsilon(g)$, a small label error $\alpha$, and smaller magnitude of ``residual eigenvalues'', \ie $\{\lambda_{k+1},\dots,\lambda_{N_1}\}$. Here, residual eigenvalues represent the high-frequency components of the augmentation graph $\gG_A$ that cannot be fitted by $k$-dimensional features. The vanilla autoencoding error $\varepsilon(g)$ depends on the non-degeneracy of the decoder $g$, while the label error $\alpha$ and residual eigenvalues $\{\lambda_{k+1},\dots,\lambda_{N_1}\}$ are purely reliant on the augmentation graph $\gG_A$ induced by the applied mask. Consequently, overall speaking, a capacious MPT model can have reduced downstream errors with a diminished vanilla autoencoding error $\varepsilon(g)$. In the meantime, the masking ratio $\rho$ should be properly chosen to ensure a small label error $\alpha$ and small magnitude of residual eigenvalues $\{\lambda_{k+1},\dots,\lambda_{N_1}\}$. In fact, as demonstrated by \cite{mae}, the masking ratio exerts a pivotal sway over the downstream efficacy of MAE. Therefore, in the next part, we explore how the choice of mask ratio would affect the downstream generalization of MPT via the label error $\alpha$ and residual eigenvalues $\{\lambda_{k+1},\dots,\lambda_{N_1}\}$.

\subsection{Theoretical Analysis on the Effect of Mask Ratio}
\label{sec:mask-empirical}

To characterize the effects of the mask ratio $\rho$ on label error and graph connectivity, we analyze a binary classification task on a spatially structured simulated data set. As illustrated in Figure \ref{fig:toy-model}, each image consists of $K$ patch positions arranged on a two-dimensional grid. For class $C_i$ and position $k$, the patch $x_k$ is sampled from a position-specific candidate set $P_{i,k}=E_{i,k}\cup O_k$, where $E_{i,k}$ contains $N_k$ class-specific patches and $O_k$ contains $M_k$ patches shared by the two classes. The candidate identities vary across positions, thereby incorporating spatial locality into the image construction. To deliver a mathematically tractable model, we further assume $N_k=N$ and $M_k=M$ for each $k$. 
During masked pretraining, for a mask $m$ with mask ratio $\rho$, we select $\rho K$ positions to form the masked view, while the remaining $(1-\rho)K$ positions constitute the unmasked view. Using this formulation, we explicitly construct the mask-induced augmentation graph and analyze how $\rho$ influences the label error $\alpha$, the eigenvalues $\lambda_1,\lambda_2,\ldots,\lambda_{N_1}$, and the generalization bound.

\begin{figure}[!htbp]
    \centering
    \includegraphics[width=.7\textwidth]{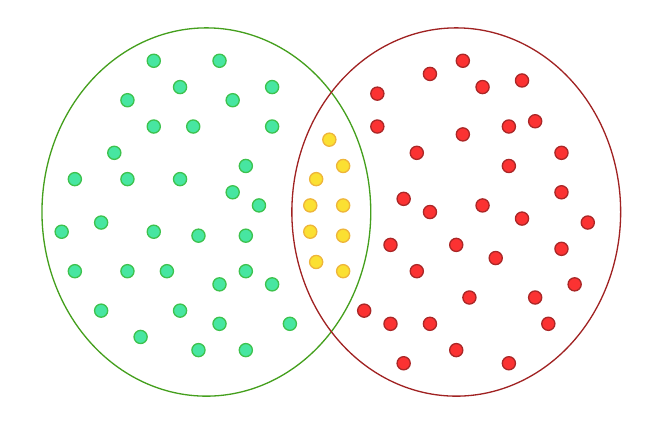}

    \caption{
    An illustration of the toy model. Each dot represents a patch in real-world data. The green and red dots correspond to non-overlapped patches from different classes, while the yellow dots are overlapped patches.\\
    }
    \label{fig:toy-model}
\end{figure}

\begin{theorem}
Let $r=\frac{M}{N+M}$ be the ratio of shared patches in each class, and $\rho$ the mask ratio satisfying $0<\rho<1$. When $r^K<\frac{1}{12}$ and $0.5\leq\rho\leq1$, with a larger mask ratio, the label error ($\alpha$) increases monotonically, the residual eigenvalues ($\sum\lambda_i^2$) decrease monotonically, and the downstream error bound decreases first and then increases. Besides, the upper bound for $\operatorname{Pr}(y(\bar x) \neq p_f'(\bar x))$ achieves its minimum value when $0.5 < \rho < 1$.
\label{thm:generalization on toy}
\end{theorem}
As shown in Theorem \ref{thm:generalization on toy}, the label error increases as the mask ratio $\rho$ increases. Meanwhile, the sum of squared residual eigenvalues exhibits a symmetric trend around $\rho=0.5$ and decreases as $\rho$ increases beyond $0.5$. Consequently, there exists a trade-off between the label error and the eigenvalues when $\rho>0.5$. Theorem \ref{thm:generalization on toy} further proves that the generalization bound first decreases and then increases with the mask ratio, and the optimal point lies in the intermediate range. This finding aligns with empirical results from MAE pretraining~\citep{mae}, i.e., 75\% is the best mask ratio for downstream tasks. Notably, while there is a local minimum in the generalization bound for $\rho<0.5$, this point does not achieve a lower bound than the optimal point for $\rho>0.5$.

\paraemph{Theoretical Insights.} 
As discussed above, Theorems \ref{thm:downstream and MPT} \& \ref{thm:downstream-spectrum} show the significance of having a small label error $\alpha$ and small residual eigenvalues ${\lambda_{k+1},\dots,\lambda_{N_1}}$ for good downstream MPT performance. Furthermore, Theorem \ref{thm:generalization on toy} demonstrates that the mask ratio $\rho$ has a decisive influence on both factors. On one hand, the label error $\alpha$ tends to increase with a larger mask ratio. Intuitively, $\alpha$ indicates the likelihood of accurately recovering the original class $y$ from the masked portions. Figure \ref{fig:MPT overlapping example} illustrates that small or moderate mask ratios, even a substantial one like 0.75, barely alter the class identity. Only when the mask ratio is exceptionally high, for instance, 0.95, the object's identity becomes nearly unrecognizable, such as with a car in the example. On the other hand, according to spectral graph theory \citep{chung1997spectral}, the residual eigenvalues ${\lambda_{k+1},\dots,\lambda_{N_1}}$ reflect the high-frequency aspects of the graph, possessing significant magnitudes when the graph is less interconnected (with numerous disjoint components). Consequently, a low downstream error necessitates enhanced connectivity within the augmentation graph $\gG_A$. Figure \ref{fig:MPT overlapping example} visually demonstrates that elevating the mask ratio masks out many diverse patterns, thereby augmenting similarity among remaining patches, especially within intra-class samples (e.g., the tires of two cars). Hence, a higher mask ratio effectively improves graph connectivity by diminishing sample diversity and augmenting inter-sample resemblance. This observation is also consistent with \citet{zhang2024look}, which shows that the flexible prediction targets in masked pretraining can induce richer inter-sample connections. 
Considering the impacts on both the label error $\alpha$ and the residual components ${\lambda_{k+1},\dots,\lambda_{N_1}}$, Theorem \ref{thm:generalization on toy} shows an evident trade-off in selecting the mask ratio: we should choose a mask ratio that is appropriately large to enhance graph connectivity. However, we also need to avoid excessively large mask ratios that lead to class mixture and confusion. In other words, a guiding principle emerges: the mask ratio should be chosen such that mask-induced graph connectivity primarily occurs among \textit{intra-class} samples (no harms), rather than \textit{inter-class} samples (resulting in significant label errors). 

\begin{figure}[!htbp]
    \centering
    \includegraphics[width=0.75\textwidth]{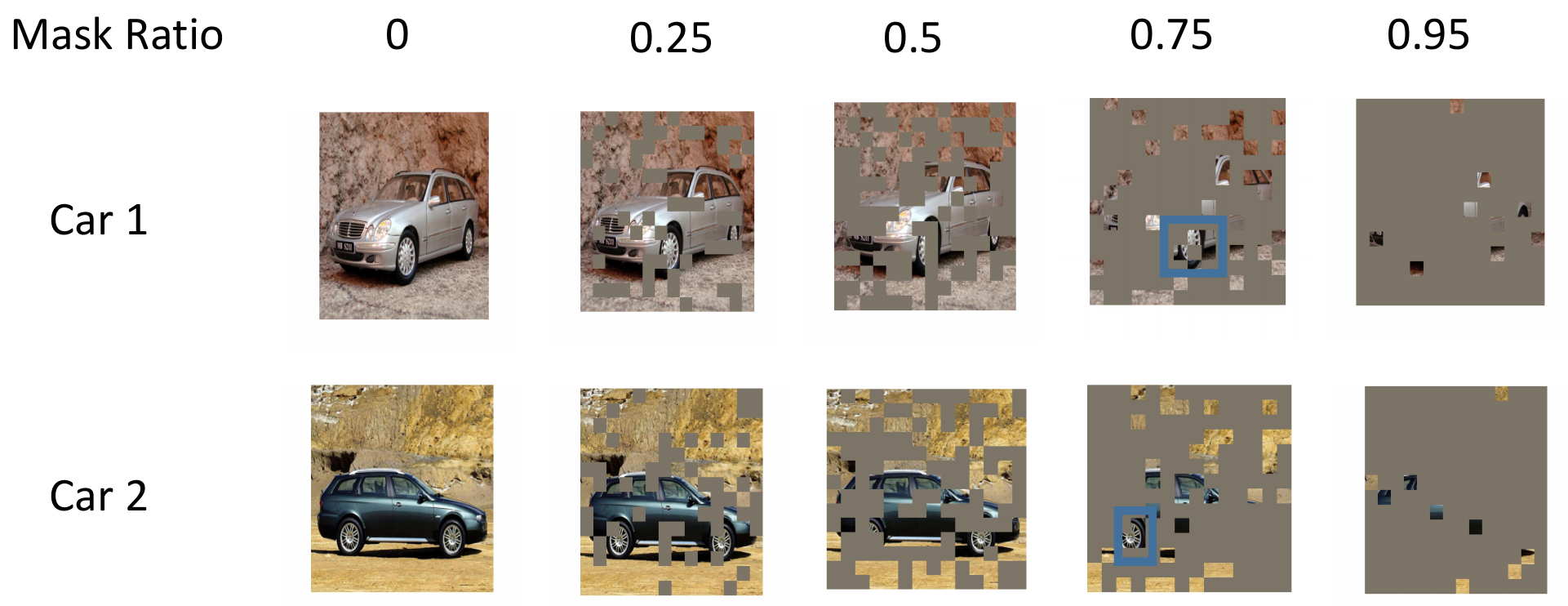}
    \caption{
    An appropriate mask ratio can generate similar views from different samples in the same class.
    }
\label{fig:MPT overlapping example}
\end{figure}

\subsection{Empirical Estimation of Downstream Errors}
\label{sec:metric}

To further validate our argument, we introduce a metric to empirically estimate the downstream error of MPT methods on real-world data sets. Since real-world images are sampled from a continuous space and two patches cannot be identical, we cannot directly estimate the probability that unmasked views share the overlapped masked view. Consequently, we first train a Vector Quantised-Variational AutoEncoder (VQ-VAE) model \citep{vqvae} and map patches into discrete tokens. Subsequently, we randomly selected 2000 images from the same data set and applied 5 random masks to each, resulting in a total of 10000 nodes within the augmentation graph. For all unmasked regions, we utilized the VQ-VAE model to map each visible patch to a unique patch ID, thus representing each view as a multiset of these IDs.

The edge weights between any two unmasked views were then estimated by calculating the degree of overlap between the sets of patch IDs from their respective masked counterparts. This degree of overlap was defined as the size of the intersection of patch IDs divided by the total number of patches in a masked view. Graph connectivity was quantified by the sum of squared residual eigenvalues, specifically $\frac{1}{N_1}\sum\limits_{i=k+1}^{N_1}\lambda_i^2$, where $k$ represents the feature dimension (set to 1000 in this case based on common practice) and $N_1$ denotes the number of nodes in the augmentation graph (equal to $10000$). The normalization coefficient $\frac{1}{N_1}$ is introduced to mitigate the impact of the number of nodes selected. Additionally, the label error was estimated by the proportion of unmasked-view pairs that exhibited over 90\% patch overlap yet possessed different ground-truth labels. The empirical results are illustrated in Figure~\ref{fig:connectivity and label error on real datasets}.

\begin{figure}[!htbp]
    \centering
    \subfigure[Residual eigenvalues]{
        \includegraphics[width=0.27\textwidth]{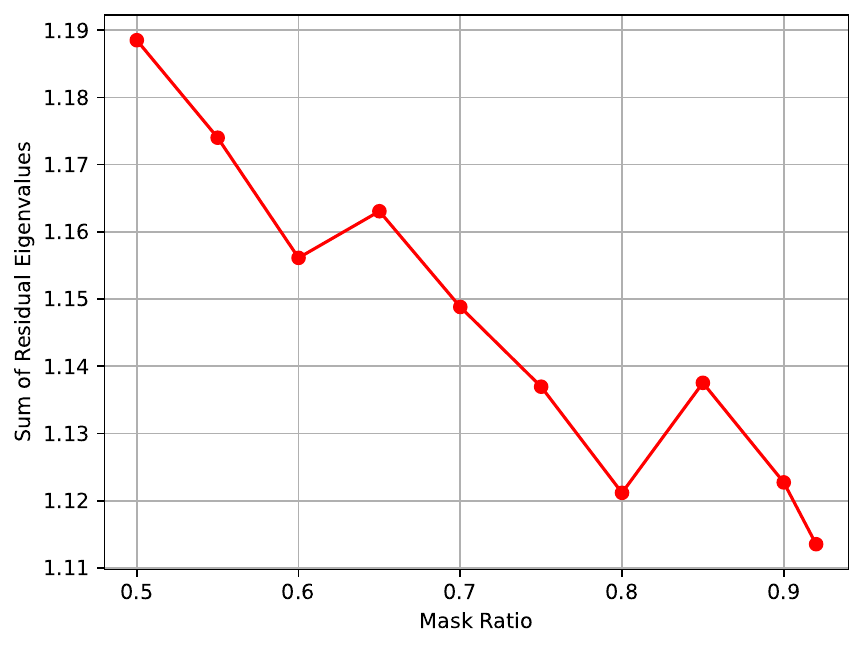}
        \label{fig:residual eigenvalues real dataset}
    }
    \hspace{0.02\textwidth}
    \subfigure[Label Error]{
        \includegraphics[width=0.27\textwidth]{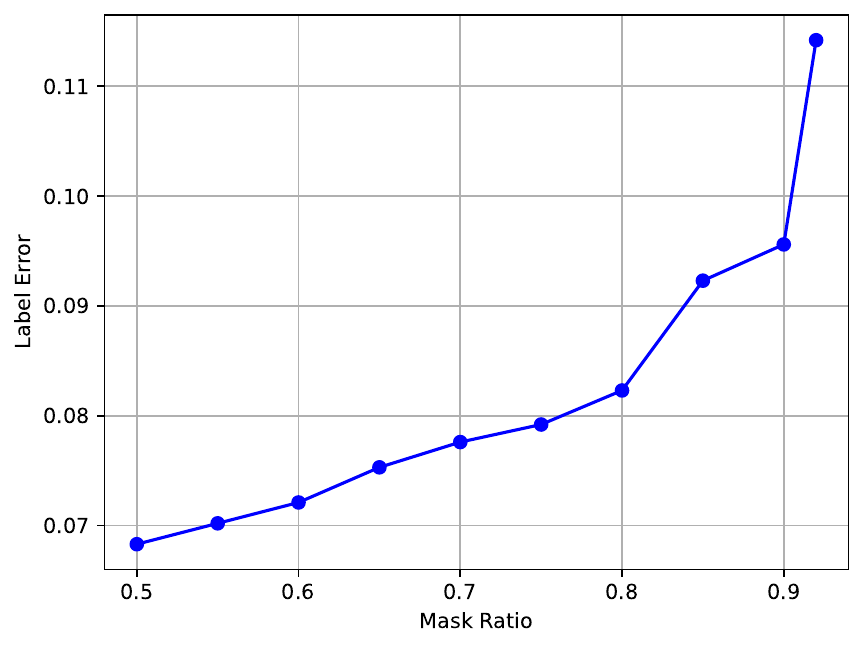}
        \label{fig:label error real dataset}
    }
    \hspace{0.02\textwidth}
    \subfigure[Tradeoff]{
        \includegraphics[width=0.27\textwidth]{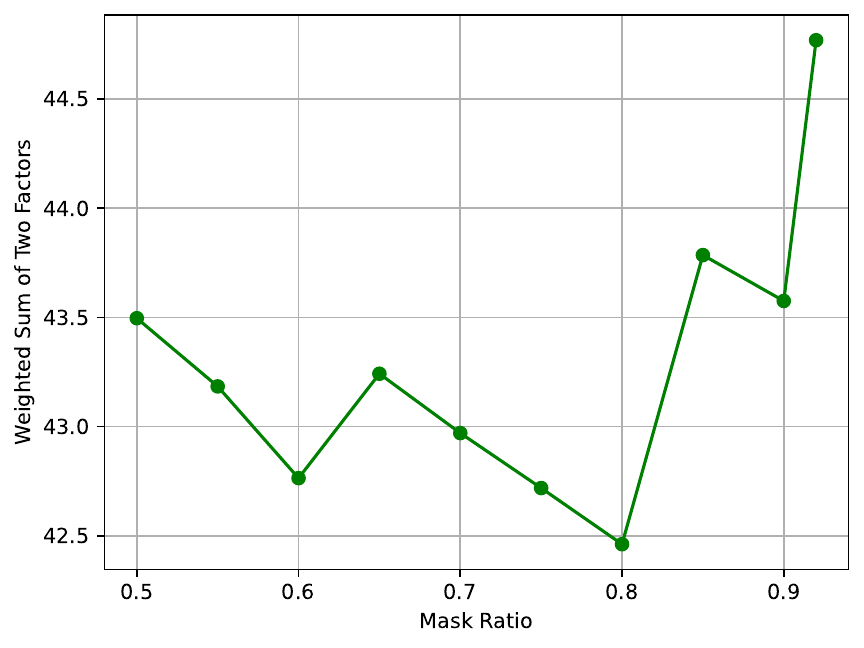}
        \label{fig:tradeoff real dataset}
    }
    \caption{(a)(b) Estimated residual eigenvalues of the augmentation graph and label error for 2000 randomly selected images from the ImageNet-100 data set, with each image subjected to 5 different random masks, under varying mask ratios.\\
    (c) The blended value of these two factors.}
    \label{fig:connectivity and label error on real datasets}
\end{figure}

The results demonstrate a steady decline in the residual eigenvalues as the mask ratio increases from 0.5, which indicates that a higher mask ratio effectively improves the connectivity of the augmentation graph. Meanwhile, the label error shows a general upward trend, which is a result of growing vagueness as more and more information is masked out. To quantify their respective contributions, we adopt the coefficients given in Section~\ref{sec:proof-of-thm-downstream-and-MPT}, \ie $\frac{32}{N_1}\sum_{i=k+1}^{N_1}\lambda_i^2+80\alpha$, and find that the weighted sum reaches its smallest value when the mask ratio is around $0.8$, very close to the value employed in~\cite{mae}. These observations align with our theoretical analysis, highlighting the inherent trade-off between the graph connectivity and the label error.

\section{Theory Inspired Strategies for Performance Improvement of MPT}

As discussed in the theoretical framework, we introduce a principle to improve the performance of MPT, i.e., we should enhance connectivity between unmasked views while avoiding label errors. Based on that, we propose new theory-guided strategies to improve the performance of MPT and analyze current variants of MPT from this perspective. In Section~\ref{sec:reweighting}, we present a novel approach that involves reassigning weights for each patch in the reconstruction loss to enhance the connectivity of the augmentation graph. In Section~\ref{sec:cur-improve-strategies}, we review existing strategies with our theoretical principles.

\subsection{Patch Re-weighting}
\label{sec:reweighting}

In original  MPT approaches, the weight for each masked patch in the reconstruction loss is uniform. However, as analyzed in contrastive learning, there exist intra-class samples that are difficult to be pulled together in the feature space and these samples usually have a greater impact on forming the connectivity of the augmentation graph \citep{zhang2025understanding}. Inspired by this, we consider paying more attention to the patches that are hard to reconstruct in MPT, namely the patches that are difficult to implicitly align. 

\begin{table}[!h]
\vspace{0.1cm}
\renewcommand{\arraystretch}{1.2}
\centering
\setlength{\tabcolsep}{3.0mm}{
\begin{tabular}{llll}
\hline
Setting & Data Set & Vanilla MAE & Re-weighted MAE \\\hline
In-domain 
& ImageNet-1K & 55.4 & \textbf{58.6} \\ \hline

\multirow{17}{*}{OOD}
& ImageNet-A & 0.57 & \textbf{0.87} \\
& ImageNet-C (Defocus Blur) & 10.48 & \textbf{20.51} \\
& ImageNet-C (Motion Blur) & 16.36 & \textbf{24.28} \\
& ImageNet-C (Frosted Glass Blur) & 13.68 & \textbf{19.03} \\
& ImageNet-C (Zoom Blur) & 14.24 & \textbf{19.21} \\
& ImageNet-C (Elastic Transform) & 23.02 & \textbf{30.94} \\
& ImageNet-C (Contrast) & 9.76 & \textbf{20.41} \\
& ImageNet-C (Brightness) & 28.52 & \textbf{46.42} \\
& ImageNet-C (Pixelate) & 24.66 & \textbf{25.50} \\
& ImageNet-C (Snow) & 9.86 & \textbf{16.93} \\
& ImageNet-C (Fog) & 10.52 & \textbf{20.69} \\
& ImageNet-C (Frost) & 11.30 & \textbf{22.00} \\
& ImageNet-C (JPEG Compression) & 20.36 & \textbf{29.41} \\
& ImageNet-C (Gaussian Noise) & 11.46 & \textbf{14.82} \\
& ImageNet-C (Impulse Noise) & 8.78 & \textbf{9.81} \\
& ImageNet-C (Shot Noise) & 11.40 & \textbf{13.60} \\
& Mean over ImageNet-C & 14.96 & \textbf{22.24} \\ \hline
\end{tabular}}
\caption{Results (\%) of vanilla MAE and Re-weighted MAE with $\gamma=1.5$, where OOD means out-of-distribution. Both models are pretrained on ImageNet-1K, and the learned representations are evaluated via linear probing on ImageNet-1K for in-domain performance and on ImageNet-A and ImageNet-C for out-of-distribution performance.}
\label{table:reweighted-results}
\end{table}

Specifically, we propose reweighted MAE, where the weight for a patch in the reconstruction loss is dynamically adjusted based on the reconstruction loss of that specific patch. Let \( h \) denote the autoencoder, \( x_1 \) the unmasked view, and \( x_2 \) the masked view. Here, \( h(x_1) \) represents the reconstructed masked view given \( x_1 \), while \( h(x_1)_k \) and \( x_{2,k} \) denote the \( k \)-th patch of \( h(x_1) \) and \( x_2 \), respectively. The reconstruction loss for the \( k \)-th patch is defined as:
\[
l_k(x_1, x_2) = \|h(x_1)_k - x_{2,k}\|_2^2.
\]
The original MAE reconstruction loss for \( x_1 \) and \( x_2 \) is then formulated as:
\[
\gL_{\textnormal{MAE}}(h, x_1, x_2) = \sum_{k} l_k(x_1, x_2).
\]
To introduce patch re-weighting, we define the weight for the \( k \)-th patch as:
\[
w_k = \left[\operatorname{stopgrad}(l_k(x_1, x_2))\right]^\gamma,
\]
where \( \operatorname{stopgrad}(\cdot) \) denotes the stop gradient operation, ensuring that \( w_k \) is detached from the back-propagation process. The reweighted loss function is then given by:
\[
\gL_{\textnormal{Re-weighted}} = \frac{1}{\sum_k w_k} \sum_k w_k \cdot l_k(x_1, x_2).
\]

To evaluate the effectiveness of re-weighted MAE, we conduct a series of experiments using a ViT-Base encoder pretrained on ImageNet-1K. As shown in Table~\ref{table:reweighted-results}, re-weighted MAE with $\gamma=1.5$ improves the in-domain linear probing accuracy on ImageNet-1K from $55.4\%$ to $58.6\%$. The improvement is more pronounced under distribution shifts, where the mean accuracy across the fifteen ImageNet-C corruption types increases from $14.96\%$ to $22.24\%$. These results suggest that emphasizing harder-to-reconstruct patches can improve both the quality and robustness of the learned representations.

We further conduct an ablation study to investigate the effect of the re-weighting coefficient $\gamma$ on ImageNet-100. As shown in Figure~\ref{fig:reweighted-acc}, the linear probing accuracy first increases as $\gamma$ grows and then decreases when $\gamma$ becomes too large, with the best performance achieved at $\gamma=1.5$. Specifically, $\gamma=1.5$ improves the linear probing accuracy by $1.98\%$ over vanilla MAE (i.e., $\gamma=0$). This trend suggests that moderately emphasizing harder-to-reconstruct patches is beneficial, whereas overly large weights may overemphasize a small subset of difficult patches and hurt the overall representation quality.

\begin{figure}[!htbp]
    \centering
    \includegraphics[width=0.6\linewidth]{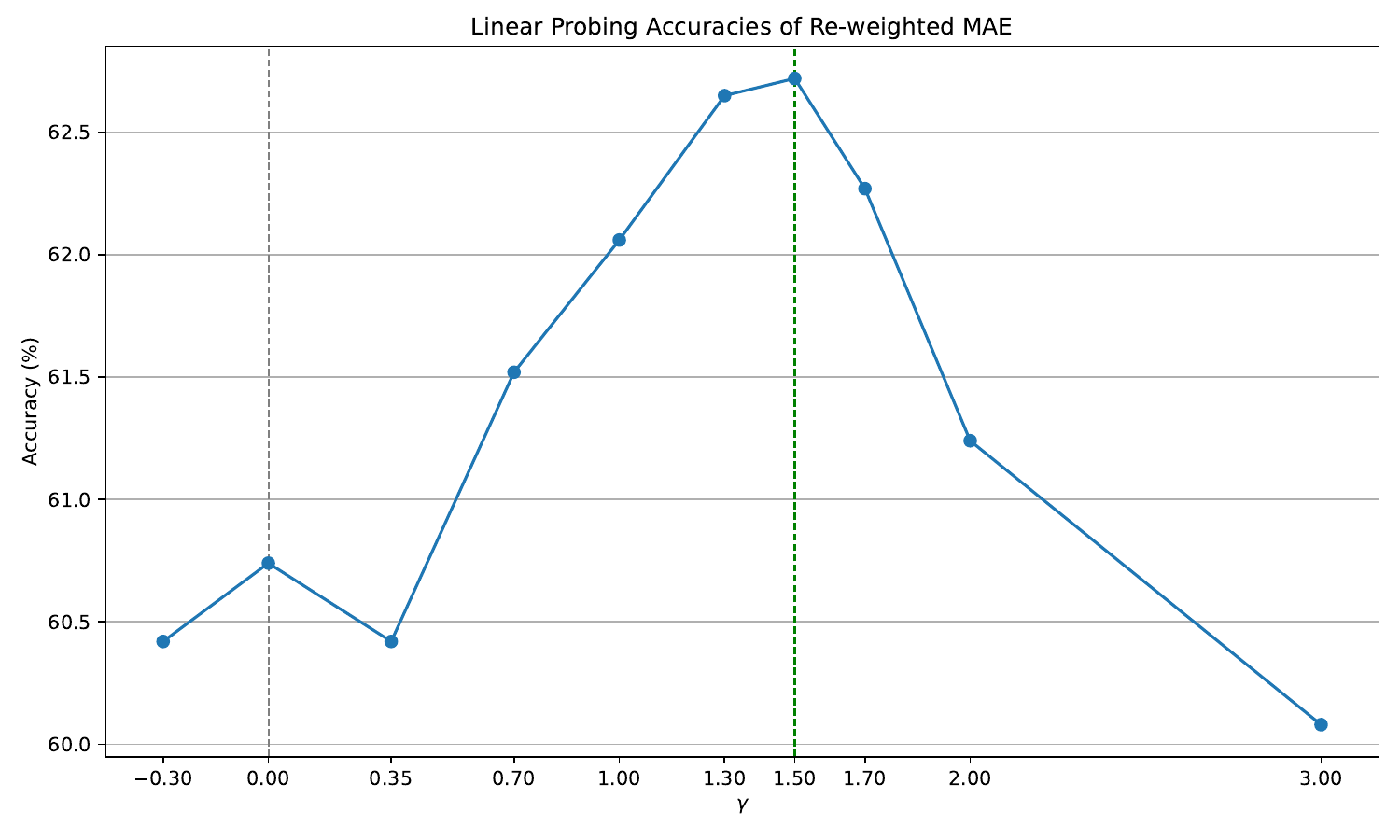}
    \caption{Linear probing accuracies of ViT-Base pretrained using Re-weighted MAE
    on the ImageNet-100 data set for 200 epochs.}
    \label{fig:reweighted-acc}
\end{figure}

\subsection{Explaining Existing Strategies}
\label{sec:cur-improve-strategies}

Recent improvements to MPT broadly fall into two categories: refinements to masking strategies and integration with contrastive learning.
We find that most of these improvements can be explained with our theoretical framework (Section~\ref{sec:downstream}), which emphasizes the roles of label error ($\alpha$) and augmentation graph connectivity ($\sum \lambda_{k+1}^2$).

\paraemph{Masking Strategy Improvements.} 
A prominent line of work optimizes masking by prioritizing \emph{salient patches}---regions critical for downstream tasks. 
For instance, \citet{Liu2023GoodHi} proposed guided masking using attention scores from the CLS token in Vision Transformers. Periodically, unmasked images are passed through the encoder to compute an importance map, where higher attention scores indicate salient patches (e.g., foreground objects). These patches are then masked with higher probability, while less important regions (e.g., backgrounds) are retained as encoder inputs. \citet{Wang2023HardPM} dynamically adjust masking based on the reconstruction difficulty of the patches and preferentially mask the core objects. Expanding this concept to relational learning, \citet{Yang2023MRMMR} introduced masked relation modeling for medical images. This approach leverages the self-attention mechanism to detect highly interdependent regions within the input image and disrupts these relationships by masking the most critical patches associated with a specific region. These methods align well with our theoretical framework. {By masking salient patches, they reduce the distinctness of masked views, generating more intra-class edges in the augmentation graph $\gG_A$. This lowers residual eigenvalues ($\sum \lambda_{k+1}^2$) and improves downstream performance as Theorem~\ref{thm:downstream-spectrum}.}

\paraemph{Contrastive Learning Integration.} 
Another category of improving strategy is to combine MPT with contrastive objectives. \citet{CMAE} proposed CMAE, which exemplifies this through a dual-branch architecture. 
The first branch follows the standard MAE framework: an encoder processes masked and pixel-shifted images, and a lightweight decoder reconstructs pixels. The second branch introduces a momentum encoder updated via exponential moving average (EMA), which processes the unshifted and unmasked version of inputs. A contrastive loss aligns features from the EMA encoder with those decoded by the first branch, effectively distilling knowledge between masked and unmasked views. The authors report a $5.9\%$ improvement in ImageNet-1K linear probing accuracy over MAE, attributing this to reduced intra-class variance and enhanced inter-class diversity in learned features. \citet{Jiang2023LayerGP} proposed Layer-Grafted Pretraining for Vision Transformers. 
Early layers are trained with MIM to capture low-level textures, while later layers employ contrastive learning to enhance semantic discriminability. This is achieved iteratively: after pretraining early layers with reconstruction loss, they are frozen, and contrastive objectives are applied to later layers. The authors observed improved separability between different classes following the application of Layer-Grafted Pretraining, based on a qualitative evaluation of the t-SNE visualizations of the representations. {We note that MPT generates intra-class edges by masking techniques while contrastive learning generates edges by data augmentation. Consequently, combining these two objectives can generate more edges and further enhance the connectivity of the augmentation graph, which indicates better downstream performance.}

\section{Conclusion}

In this work, we established a theoretical connection between masked pretraining (MPT) and contrastive learning, showing that minimizing the MPT loss implicitly aligns mask-induced positive pairs. We identified dimensional collapse as a key limitation of MPT and proposed U-MPT, a uniformity-enhanced variant that explicitly promotes feature diversity. Theoretically, we derived downstream generalization guarantees, explaining the empirical success of high mask ratios (e.g., $\rho=0.75$) through their balance of intra-class connectivity and inter-class discriminability. Empirically, U-MPT achieved superior robustness on out-of-distribution data sets like ImageNet-A and ImageNet-C. Additionally, based on our theoretical analysis, we propose a new reconstruction strategy that prioritizes challenging regions during reconstruction and enhances the performance of MPT. These contributions provide a unified framework for optimizing MPT, bridging reconstruction fidelity with feature discriminability.

\acks{Yisen Wang was supported by National Natural Science Foundation of China (92370129, 62376010), Beijing Major Science and Technology Project under Contract no. Z251100008425006, Beijing Natural Science Foundation (L257007), Beijing Nova Program (20230484344, 20240484642), and State Key Laboratory of General Artificial Intelligence.}

\newpage
\appendix

\section{Proofs}

In the following proofs, given the fact that $g$ is fixed, we write $\varepsilon=\varepsilon(g)$ for notational simplicity.

{
\subsection{Proof of Theorem \ref{thm:implicit-alignment}}

\begin{proof}
With Assumption \ref{ass:aprroximation pesudo encoder}, we have
\begin{align*}
    \gL_\text{MAE}(h)&  = \E_{x_1,x_2} \| h(x_1) - x_2 \|^2 \\
    &=\E_{x_1,x_2} \| h(x_1) - x_2\|^2 +\varepsilon - \varepsilon\\
    &\geq\E_{x_1,x_2} \| h(x_1) - x_2\|^2 +\Vert x_2-h_g(x_2) \Vert ^2 - \varepsilon &(\Vert x_2-h_g(x_2) \Vert ^2\leq \varepsilon)\\
    &\geq \frac{1}{2}\E_{x_1,x_2} \| h(x_1) - h_g(x_2)\|^2 -\varepsilon &(\Vert a+b \Vert ^2 \leq 2 (\Vert a\Vert^2 + \Vert b\Vert^2)  \\
    &= -\E_{x_1,x_2}  h(x_1)^\top h_g(x_2) -\varepsilon + 1.  &(h(x) \text{ and } h_g(x) \text{ are normalized}) \\
\end{align*}
We formulate the features as two matrices $H$,$H_g$. We denote $H(x_1)  =  \sqrt {d_{x_{1}}}h(x_{1})$ as the $x_1$-th row of the matrix $H$ and $H_g(x_2)  =  \sqrt {d_{x_{2}}}h_g(x_{2})$ as the $x_2$-th row of the matrix $H_g$. As defined before, $(A_M)_{x_2,x_1} $ is the joint distribution of $x_{1}$ and $x_{2}$, \ie $(A_M)_{x_2,x_1} = w_{x_1,x_2}$. We denote the normalized form of $A_M$ as $\bar{A}_M$, \ie $(\bar{A}_M)_{x_2,x_1}=\frac{w_{x_1,x_2}}{\sqrt{d_{x_1 }}\sqrt{d_{  x_2} }}$. Then we can reformulate the reconstruction loss,
\begin{align*}
    \gL_\text{MAE}(h)
    &\geq -\sum\limits_{x_1,x_2}\frac{w_{x_1,x_2}}{\sqrt{d_{x_1}d_{ x_2}}}  \sqrt{d_{x_1}}h(x_1)\cdot \sqrt {d_{ x_2}}h_g(x_2) -\varepsilon + 1\\
    &= -\tr(\bar{A}_MHH_g^\top) -\varepsilon + 1   \\
    &\geq -\frac{1}{2}(\|\bar{A}_MH\|^2+\|H_g\|^2) -\varepsilon + 1 .\\
\end{align*}
As the output of decoder is normalized, \ie $\|H_g\|^2 = 1$. we obtain
\begin{equation}
    \gL_\text{MAE}(h)\geq-\frac{1}{2}(\tr(\bar{A}_M^\top \bar{A}_MHH^\top) + 1) - \varepsilon + 1=-\frac{1}{2}\tr(\bar{A}_M^\top \bar{A}_MHH^\top)  - \varepsilon + \frac{1}{2}.\\
\label{equ:matrix-rec}
\end{equation}
Then we element-wise compute $\bar{A}_M^\top \bar{A}_M$ and $HH^\top$, we have
\begin{equation}
(\bar{A}_M^\top \bar{A}_M)_{x_1,x_1^+} =  \sum\limits_{x_2} \frac{w_{x_1,x_2}w_{x_1^+,x_2}}{d_{x_2}\sqrt{d_{x_1 }d_{ x_1^+ }}}.
\end{equation}
\begin{equation}
(HH^\top)_{x_1^+,x_1} = \sqrt{d_{ x_1}d_{\ x_1^+}}h(x_1)^\top h(x_1^+).
\end{equation}
As the trace is the sum of the diagonal value of the matrix, we consider $x_1$-th diagonal value of $(\bar{A}_M^\top \bar{A}_M H H^\top)$,
\ie
\begin{equation}
(\bar{A}_M^\top \bar{A}_M H H^\top)_{x_1,x_1} = \sum\limits_{x_1^+} (A_M^\top A_M)_{x_1,x_1^+} (HH^\top)_{x_1^+,x_1} = \sum\limits_{x_1^+} \sum\limits_{x_2} \frac{w_{x_1,x_2}w_{x_1^+,x_2}}{d_{x_2} } h(x_1)^\top h(x_1^+). 
\end{equation}
With that, we can element-wise expand Eq.~(\ref{equ:matrix-rec}),
\begin{align*}
 \gL_\text{MAE}(h)&\geq-\frac{1}{2}\E_{x_1,x_1^+ \sim \hat{p}(x,x^+)}h(x_1)^\top h(x_1^+) -\varepsilon + \frac{1}{2},
\end{align*}
where $\hat{p}(x_1,x_1^+) = \sum\limits_{x_2} \frac{w_{x_1,x_2}w_{x_1^+,x_2}}{d_{x_2} }$.
Similarly, we define a matrix $H_g$, where $(H_g)_{x_2} = \sqrt{d_{x_{2}}}h_g(x_{2}) $. We let $(\bar{A}_M)_{x_2,x_1}=\frac{w_{x_1,x_2}}{\sqrt{d_{ x_1}}\sqrt{d_{ x_2}}}$ and we obtain
\begin{align*}
    \gL_\text{MAE}(h)
    &\geq -\tr(HH_g^\top \bar{A}_M) -\varepsilon + 1   \\
    &\geq -\frac{1}{2}(\|H\|^2+\|\bar{A}_M^\top H_g\|^2) -\varepsilon + 1  & (tr(AB) \leq \frac{1}{2}\Vert A \Vert ^2 + \Vert B \Vert^2)\\
    &=-\frac{1}{2} (\tr(\bar{A}_M \bar{A}_M^\top H_gH_g^\top) + 1) -\varepsilon + 1 &\text{($H_g$ is normalized)}\\
    &\geq-\frac{1}{2}\E_{x_2^+ \sim \bar{p}(x_2,x_2^+)}h_g(x_2)^\top h_g(x_2^+) -\varepsilon + \frac{1}{2},
\end{align*}
where $\bar{p}(x_2,x_2^+) = \sum\limits_{x_1} \frac{w_{x_1,x_2}w_{x_1,x_2^+}}{d_{x_1} }$.
\end{proof}

}
\subsection{Proof of Corollary \ref{cor:implicit-alignment}}
\begin{proof}
With Theorem \ref{thm:implicit-alignment}, we know that
\begin{align*}
 \gL_\text{MAE}(h)&\geq-\frac{1}{2}\E_{x_1,x_1^+ \sim \hat{p}(x,x^+)}h(x_1)^\top h(x_1^+) -\varepsilon + \frac{1}{2}.
\end{align*}
As the decoder is $L$-bi-Lipschitz, we obtain
\begin{equation}
\forall\ (x_1,x_2), 1/L \cdot \| x_1 -x_2\|^2\leq\|g(x_1)-g(x_2)\|^2 \leq L \cdot \| x_1 -x_2\|^2.
\label{eqn:bili}
\end{equation}
So,
\begin{align*}
 \gL_\text{MAE}(h)&\geq-\frac{1}{2}\E_{x_1,x_1^+ \sim \hat{p}(x_1,x_1^+)}h(x_1)^\top h(x_1^+) -\varepsilon + \frac{1}{2}\\
 &=\frac{1}{4}\E_{x_1,x_1^+ \sim \hat{p}(x_1,x_1^+)}\Vert h(x_1)- h(x_1^+) \Vert ^2 -\varepsilon  &\text{($h(x)$ is normalized)}\\
  &=\frac{1}{4}\E_{x_1,x_1^+ \sim \hat{p}(x_1,x_1^+)}\Vert g(f(x_1))- g(f(x_1^+)) \Vert ^2 -\varepsilon  \\
  &\geq\frac{1}{4L}\E_{x_1,x_1^+ \sim \hat{p}(x_1,x_1^+)}{\Vert f(x_1)-f(x_1^+) \Vert} ^2 -\varepsilon  &(\text{Equation \ref{eqn:bili}}) \\  
  &=-\frac{1}{2L}\E_{x_1,x_1^+ \sim \hat{p}(x_1,x_1^+)} f(x_1)^\top f(x_1^+)  -\varepsilon + \frac{1}{2}.  \\    
\end{align*}
\end{proof}
\subsection{Proof of Theorem \ref{thm:MPT not tc}}
\begin{proof}
 When the encoder fully collapses, the encoder $f$ maps all the features to the same point $c$, \ie
 \begin{equation}
     \forall x \in \gX_1, f(x) = c.
 \end{equation}
Then,
 \begin{equation}
\begin{aligned}
     \gL_{MAE}(h) &= \E_{x_1,x_2}\Vert g(f(x_1)) - x_2\Vert ^2\\
     &= \E_{x_2}\Vert g(c) - x_2\Vert^2.\\
\end{aligned}
\label{kkt dis}
 \end{equation}
We then select a g(c) to make Equation (\ref{kkt dis}) minimal. 
According to KKT conditions, it has a closed-form solution $q^\star$, satisfying
\begin{equation}
    2\E_{x_2}(q^\star - x_2) = 0.
\end{equation}
\ie $q^\star = \E_{x_2}x_2$. Then
\begin{align*}
     \gL_{MAE}(h)
     &= \E_{x_2}\Vert g(c) - x_2\Vert^2\\
     &\geq \E_{x_2}\Vert \E_{x_2'} x_2' - x_2\Vert^2\\
     &= \text{Var}(x_2).
\end{align*}
\end{proof}

\subsection{Proof of Theorem \ref{thm:spectral-loss}}
\begin{proof}
With Corollary \ref{cor:implicit-alignment}, we have
\begin{equation}
\gL_\text{MAE}(h)\geq-1/(2L)\cdot \E_{x_1,x_1^+} f(x_1) ^\top f(x_1^+) -\epsilon + \textnormal{const}.
\end{equation}
Then we set $\lambda = \frac{1}{4L}$ and we obtain
\begin{equation}
    \begin{aligned}
    \gL_\text{U-MPT}(h) &= \gL_{MAE}(h) + 1/(4L) \cdot \gL_{unif}(f)\\
    &\geq 1/(2L) \cdot \gL_{align}(f) + 1/(4L) \cdot \gL_{unif}(f) -\epsilon + \textnormal{const}\\
    &= 1/(4L) \cdot \gL_{SCL}(f) -\epsilon + \textnormal{const}.
    \end{aligned}
\end{equation}
\end{proof}

\subsection{Proof of Theorem \ref{thm:downstream and MPT}}
\label{sec:proof-of-thm-downstream-and-MPT}

\begin{proof}

 We compose the marginal distribution of $x_1$ as a matrix $D$ and $D_{x_1} = d_{x_1}$ is the $x_1$-th row of $D$. And we denote $U$ as the matrix 
composed of encoder features, \ie $U_{x_1} = \sqrt{ d_{x_1}} f(x_1)$. Recall that  $A_{x_1,x_1^+} = \gA(x_1,x_1^+)$ and $\bar{A}$ is the normalized form of $A$, \ie $\bar{A}_{x_1,x_1^+} = \frac{\gA(x_1,x_1^+)}{\sqrt{ d_{x_1}\cdot d_{x_1^+}}}$.
Then we reformulate the downstream error,
\begin{align*}
\E_{x,y}\|y-W_ff(x)\|^2&= \sum\limits_{(x_1,y_{x_1})}d_{x_1}\| y_{x_1} - W_f f(x_1)\|^2 \\
&= \|D^{1/2}Y-UW_f^\top\|^2 \\
&= \|D^{1/2}Y-\bar{A}C + \bar{A}C - UW_f^\top\|^2, \\
\end{align*}
where $C_{x_1,j} = \sqrt{(d_i)\mathbbm{1}_{y_{x_1}=j}}$.
Then we consider the relationship between the downstream error and the augmentation graph, we element-wise consider the matrix $(D^{1/2}Y-\bar{A}C)$,
\begin{equation}
    (D^{1/2}Y)_{x_1,j} = \sqrt{(d_{x_1})\mathbbm{1}_{y_{x_1}=j}}, \quad (\bar{A}C)_{{x_1},j} = \sum\limits_{x_1^+} \frac{w_{{x_1},x_1^+}}{\sqrt{d_{x_1}}\cdot \sqrt{d_{x_1^+}}} \sqrt{(d_{x_1^+})\mathbbm{1}_{y_{x_1^+}=j}}.
\end{equation}
So when $j = y_{x_1}  $,
\begin{equation}
\begin{aligned}
    (D^{1/2}Y-\bar{A}C)_{{x_1},j} &= \sqrt{(d_{x_1})\mathbbm{1}_{y_{x_1}=j}} -\sum\limits_{x_1^+} \frac{\gA(x_1,x_1^+)}{\sqrt{d_{x_1}}} 
    \mathbbm{1}_{y_{x_1^+}=j} \\
    &= \sqrt{d_{x_1}} -\sum\limits_{x_1^+} \frac{\gA(x_1,x_1^+)}{\sqrt{d_{x_1}}} \mathbbm{1}_{y_{x_1^+}=j} \\ 
    &= \sum\limits_{x_1^+}\frac{{\gA(x_1,x_1^+)}}{\sqrt{d_{x_1}}} -\sum\limits_{x_1^+} \frac{\gA(x_1,x_1^+)}{\sqrt{d_{x_1}}} \mathbbm{1}_{y_{x_1^+}=j} \\
    &= \sum\limits_{x_1^+} \frac{\gA(x_1,x_1^+)}{\sqrt{d_{x_1}}} \mathbbm{1}_{y_{x_1^+}\neq j}\\
    &= \sum\limits_{x_1^+} \frac{\gA(x_1,x_1^+)}{\sqrt{d_{x_1}}} \mathbbm{1}_{y_{x_1^+}\neq y_{x_1}}.\\
\end{aligned}
\end{equation}
\vspace{-1 cm}
When $j \neq y_{x_1} $,
\begin{equation}
\begin{aligned}
    (D^{1/2}Y-\bar{A}C)_{{x_1},j} &= \sqrt{(d_{x_1})\mathbbm{1}_{y_{x_1}=j}} -\sum\limits_{x_1^+} \frac{\gA(x_1,x_1^+)}{\sqrt{{d_{x_1}}}} \mathbbm{1}_{y_{x_1^+}=j} \\
     &= 0 -\sum\limits_{x_1^+} \frac{\gA(x_1,x_1^+)}{\sqrt{d_{x_1}}} \mathbbm{1}_{y_{x_1^+}=j} \\
    &= -\sum\limits_{x_1^+} \frac{\gA(x_1,x_1^+)}{\sqrt{d_{x_1}}} \mathbbm{1}_{y_{x_1^+} = j}.
\end{aligned}
\end{equation}
\begingroup
\raggedbottom
\allowdisplaybreaks[4]
We define $\beta_{x_1} = \sum\limits_{x_1^+} \gA(x_1,x_1^+)\mathbbm{1}_{y_{x_1^+} \neq y_{x_1}} $, and we have
\begin{equation}
\begin{aligned}
\Vert( D^{1/2}Y-\bar{A}C )_{x_1}\Vert^2 &= \left(\sum\limits_{x_1^+} \frac{\gA(x_1,x_1^+)}{\sqrt{d_{x_1}}} \mathbbm{1}_{y_{x_1^+}\neq y_{x_1}}\right)^2 + \sum\limits_{j\neq y_{x_1}}\left(\sum\limits_{x_1^+} \frac{\gA(x_1,x_1^+)}{\sqrt{d_{x_1}}} \mathbbm{1}_{y_{x_1^+} = j}\right)^2 \\
&\leq\left(\sum\limits_{x_1^+} \frac{\gA(x_1,x_1^+)}{\sqrt{d_{x_1}}} \mathbbm{1}_{y_{x_1^+}\neq y_{x_1}}\right)^2 + \left(\sum\limits_{j\neq y_{x_1}} \sum\limits_{x_1^+} \frac{\gA(x_1,x_1^+)}{\sqrt{d_{x_1}}} \mathbbm{1}_{y_{x_1^+} = j}\right)^2\\
&\leq\left(\sum\limits_{x_1^+} \frac{\gA(x_1,x_1^+)}{\sqrt{d_{x_1}}} \mathbbm{1}_{y_{x_1^+}\neq y_{x_1}}\right)^2 + \left( \sum\limits_{x_1^+} \frac{\gA(x_1,x_1^+)}{\sqrt{d_{x_1}}} \sum\limits_{j\neq y_{x_1}} \mathbbm{1}_{y_{x_1^+} = j}\right)^2\\
&=\left(\sum\limits_{x_1^+} \frac{\gA(x_1,x_1^+)}{\sqrt{d_{x_1}}} \mathbbm{1}_{y_{x_1^+}\neq y_{x_1}}\right)^2 + \left( \sum\limits_{x_1^+} \frac{\gA(x_1,x_1^+)}{\sqrt{d_{x_1}}} \mathbbm{1}_{y_{x_1^+} \neq y_{x_1}}\right)^2\\
&= \frac{2\beta_{x_1}^2}{d_{x_1}}.
\end{aligned}
\end{equation}
With that, we obtain $\Vert D^{1/2}Y-\bar{A}C \Vert^2 =  \sum\limits_{{x_1}} \frac{ 2\beta_{x_1}^2}{d_{x_1}}.$
Let $\beta=\sum_{x_1,{x_1^+}}\gA(x_1,x_1^+)\mathbbm{1}_{y_{x_1^+} \neq y_{x_1}}$. Then we have
\begin{align*}
\Vert D^{1/2}Y-\bar{A}C \Vert^2 &=  \sum\limits_{x_1} \frac{ 2\beta_{x_1}^2}{d_{x_1}} \\
&\leq     \sum\limits_{x_1} 2\beta_{x_1} &\left(d_{x_1} = \sum\limits_{x_1^+}\gA(x_1,x_1^+)\geq \beta_{x_1}\right) \\
&=2\beta.
\end{align*}
Then we obtain
\begin{align}
\E_{x,y}\|y-W_ff(x)\|^2
&=\|D^{1/2}Y-\bar{A}C+\bar{A}C-UW_f^\top\|^2
\notag\\
&\leq 2\|\bar{A}C-UW_f^\top\|^2+4\beta
\notag\\
&=2\|(\bar{A}-UU^T+UU^T)C-UW_f^\top\|^2+4\beta
\notag\\
&=2\|(\bar{A}-UU^T)C+U(U^TC-W_f^\top)\|^2+4\beta
\notag\\
&\leq
4\bigl(
\|(\bar{A}-UU^T)C\|^2
+\|U(U^TC-W_f^\top)\|^2
\bigr)+4\beta
\notag\\
&\leq
4\bigl(
\|\bar{A}-UU^T\|^2\|C\|^2
+\|U\|^2\|U^TC-W_f^\top\|^2
\bigr)+4\beta
\notag\\
&\leq
4\bigl(
\|\bar{A}-UU^T\|^2
+\|U^TC-W_f^\top\|^2
\bigr)+4\beta
\notag\\
&=
4\gL_{SCL}(f)
+\textnormal{const}
+\|U^TC-W_f^\top\|^2
+4\beta
\notag\\
&\leq
16L\cdot\gL_{\text{U-MPT}}(h)
+8\E_y[\E_{x|y}f(x)-b_y]^2
+4\beta
+16L\varepsilon
+\textnormal{const}
\notag\\
&\leq
16L\cdot\gL_{\text{U-MPT}}(h)
+4\beta
+16L\varepsilon
+\textnormal{const}.
\label{eqn:umae_and_scl}
\end{align}
\endgroup
Here we leverage Lemma B.8 in \cite{haochen} and the fact that $W_f$ is a mean classifier.
In the next step, we analyze the prediction error.
We denote $y(\bar{x})$ as the ground-truth label of original data $\bar{x}$.
We first define an ensembled linear predictor $c_f'$. For an original sample, the predictor ensembles the results of all different views and choose the label predicted most. With the definition, $y(\bar{x}) \neq c_f'(\bar{x})$ only happens when more than half of the views predict wrong labels.
So
\begin{equation}
\begin{aligned}
    &\operatorname{Pr}(y(\bar{x})\neq c_f'(\bar{x})) \\
    &\leq 2 \operatorname{Pr}(y(\bar{x})\neq c_f(x))\\
    &\leq 4\E_{\bar{x}\sim \gP_d(x),x\sim \gM_1(x|\bar{x})}\|y(\bar{x})-W_ff(x)\|^2\\
    &\leq 8( \E_{x,y}\|y-W_ff(x)\|^2+ \E_{\bar{x}\sim \gP_d(x),x\sim \gM_1(x|\bar{x})}\|y-y(\bar{x})\|^2) \\
    &\leq  8(\E_{x,y}\|y-W_ff(x)\|^2 + 2\alpha)\\
    &\leq 32 \gL_{SCL}(f) +64\alpha + 16\alpha +\textnormal{const}\\
    &\leq  128L \cdot \gL_\text{U-MPT}(h)+32\beta +16\alpha+ 128L \varepsilon+ \textnormal{const}. 
\end{aligned}
\label{eqn:pred error alpha beta}
\end{equation}

Also note that, as we assume $E_{\bar{x}\sim \gP_d}(\gA(x_1|\bar{x})\mathbbm{1}_{y_{x_1}\neq y(\bar{x})})\leq \alpha$,
we have
\begin{equation}
\begin{aligned}
\beta
 &= \sum_{x_1,{x_1^+}} \ E_{\bar{x}}(\gA(x_1|\bar{x})\gA(x_{x_1^+}|\bar{x})\mathbbm{1}_{y_{x_1^+} \neq y_{x_1}})\\
  &\leq \sum_{x_1,{x_1^+}} \ E_{\bar{x}}(\gA(x_1|\bar{x})\gA(x_{x_1^+}|\bar{x})(\mathbbm{1}_{y_i \neq y(\bar{x})}+\mathbbm{1}_{y_{x_1^+} \neq y(\bar{x})}))\\
  &= 2E_{\bar{x}\sim \gP_d}(\gA(x_1|\bar{x})\mathbbm{1}_{y_{x_1}\neq y(\bar{x})})\\
  &\leq 2\alpha.
\end{aligned}
\label{eqn:natural error and aug eror}
\end{equation}

This implies

$$
\operatorname{Pr}(y(\bar{x})\neq c_f'(\bar{x})) \le 128L \cdot \gL_\text{U-MPT}(h)+80\alpha+ 128L \varepsilon+ \textnormal{const},
$$

which gives the values of the constants mentioned in Theorem~\ref{thm:downstream and MPT}.

\end{proof}

\subsection{Proof of Theorem \ref{thm:downstream-spectrum}}
\begin{proof}
 With Equation (\ref{eqn:umae_and_scl}), we have 
 \begin{equation}
         \gL_\text{U-MPT}(h) \geq  \frac{1}{4L} \Vert A - UU^\top \Vert ^2 - \varepsilon +\textnormal{const}.
 \end{equation}

 We set $\gL_{mf} (U) = \|(\bar{A}-UU^T) \|^2$. When $U^\star$ is the minimizer of $\gL_{mf} (U)$, according to the analysis in \cite{eckart1936approximation}, we obtain
\begin{equation}
    \|(\bar{A}-U^\star (U^\star)^T) \|^2 = \sum\limits_{i=k+1}^{N_1}\lambda_i^2 ,
\end{equation}
where $\lambda_{k+1}\cdots \lambda_{N_1}$ are the remaining $N_1-k$ eigenvalues of matrix $\bar{A}$. 
We denote $h^\star$ as the minimizer of $\gL_\text{U-MPT}$ and $f^\star$ as the corresponding encoder. Then $U_{f^\star}$ is composed of the features encoded by $f^\star$, \ie $(U_{f^\star})_{x_1} = \sqrt{d_{x_1}}f^\star(x_1)$. So for all $h\in\gH$, we have
\begin{equation}
\begin{aligned}
          \gL_\text{U-MPT}(h)\geq\gL_\text{U-MPT}(h^\star) &\geq  \frac{1}{4L} \Vert A - U_{f^\star} (U_{f\star})^\top \Vert ^2 -  \varepsilon +\textnormal{const}\\
         &\geq \frac{1}{4L}  \Vert A - U^\star (U^\star)^\top \Vert ^2 -    \varepsilon +\textnormal{const}\\
         &=\frac{1}{4L}\sum\limits_{i=k+1}^{N_1}\lambda_i^2 -  \varepsilon +\textnormal{const}.\\
\end{aligned}
\label{eqn:applambda}
\end{equation}

\end{proof}

\section{Analysis of the Toy Model}

In this section, we provide the theoretical proofs of the toy model.

\subsection{Problem Setting and Probabilistic Model}

Consider two classes, $C_1$ and $C_2$, and $K$ patch positions arranged on a two-dimensional grid. For each class $C_i$ and position $k$, define a local candidate set
\[
P_{i,k}=E_{i,k}\cup O_k,
\]
where $E_{i,k}$ contains $N$ class-specific patches and $O_k$ contains $M$ patches shared by the two classes. At the same position, the class-specific sets are disjoint, i.e., $E_{1,k}\cap E_{2,k}=\emptyset$, and the two classes overlap only through $O_k$. The identities of these candidate patches may vary with $k$. If an image belongs to class $C_i$, then its $k$-th patch is independently and uniformly sampled from $P_{i,k}$.

To create a masked view from an image, $\rho K$ positions are randomly selected as the masked portion, while the remaining $(1-\rho)K$ positions form the unmasked view, where $\rho$ is the mask ratio. Let $I_2$ denote the masked-position set and $I_1$ denote the unmasked-position set. A view is termed an overlapped view if all its patches are drawn from the shared local sets $O_k$ at the corresponding positions; otherwise, it is considered a non-overlapped view.

The following quantities come in handy:
\begin{align*}
H&=\prod_{k\in I_2}|P_{i,k}|=(N+M)^{\rho K},\\
W&=\prod_{k\in I_2}|O_k|=M^{\rho K},\\
R&=\prod_{k\in I_1}|P_{i,k}|=(N+M)^{(1-\rho)K},\\
G&=\prod_{k\in I_1}|O_k|=M^{(1-\rho)K}.
\end{align*}
where $H,W,R,G$ denote the number of possible masked views, overlapped masked views, possible unmasked views, and overlapped unmasked views, respectively.

Given the uniform sampling, the marginal probability of a view with $p$ patches being sampled from a specific class is $\frac{1}{(N+M)^p}$.
Considering both classes, we can derive
the marginal probability of each view as follows:

\begin{itemize}
    \item Non-overlapped masked views: $\frac{1}{2H}$
    \item Overlapped masked views: $\frac{1}{H}$
    \item Non-overlapped unmasked views: $\frac{1}{2R}$
    \item Overlapped unmasked views: $\frac{1}{R}$
\end{itemize}

Note that overlapped views have twice the probability of being selected as non-overlapped views, since they can be chosen by both classes.

We use $x_1,x_1^+$ to denote unmasked views, and $x_2$ masked views.
The quantity $w_{x_1,x_2}$ represents the probability that $x_1$ and $x_2$ come from the same image.
Specifically, $w_{x_1,x_2}=\Pr(x_2)\Pr(x_1|x_2)$.
Clearly, $w_{x_1,x_2} = 0$ if both $x_1$ and $x_2$ are non-overlapped and from different classes.
Otherwise, $w_{x_1,x_2}$ is equal to $\frac{1}{2HR}$ when $x_1$ or $x_2$ is non-overlapped, and $\frac{1}{HR}$ when both are overlapped.

The nodes in an augmentation graph correspond to the set of all unmasked views. The normalized edge weight between two unmasked views, $x_1$ and $x_1^+$, is given by

$$
    \gA^*(x_1,x_1^+)=\frac{1}{\sqrt{w_{x_1}w_{x_1^+}}}\sum\limits_{x_2}\frac{w_{x_1,x_2}w_{x_1^+,x_2}}{w_{x_2}}
$$
where $w_{x_1}$ represents the marginal probability of $x_1$. Here we use a normalized version of the augmentation graph for the sake of simplicity.

$\gA^*(x_1,x_1^+)$ can take one of the following values:

\begin{itemize}
    \item \textbf{One of $x_1$ and $x_1^+$ is non-overlapped and
    the other is overlapped}: \\
    $\gA^*(x_1,x_1^+)=\frac{1}{\sqrt{2}R}:=p_1$.
    \item \textbf{$x_1$ and $x_1^+$ are in the same class and both non-overlapped}: \\
    $\gA^*(x_1,x_1^+)=\frac{2H-W}{2HR}:=p_2$.
    \item \textbf{$x_1$ and $x_1^+$ are in different classes and both non-overlapped}: \\
    $\gA^*(x_1,x_1^+)=\frac{W}{2HR}:=p_3$.
    \item \textbf{Both $x_1$ and $x_1^+$ are overlapped}: \\
    $\gA^*(x_1,x_1^+)=\frac{1}{R}:=p_4$.
\end{itemize}

\subsection{Adjacency Matrix and Graph Connectivity}

The adjacency matrix of the augmentation graph, $A$, is symmetric and structured into 9 blocks, with the widths of blocks labeled above:

$$
    A = 
    \left[
        \begin{array}{c|c|c}
        \overbrace{A_{11}}^{R-G} & \overbrace{A_{12}}^{R-G} & \overbrace{A_{13}}^{G} \\
        \hline
        A_{21} & A_{22} & A_{23} \\
        \hline
        A_{31} & A_{32} & A_{33} \\
        \end{array}
    \right].
$$

Each block $A_{ij}\ (i=1,2,3;j=1,2,3)$ is a constant matrix, and the entries are given below:

$$
    C = {(c_{ij})}_{3\times 3} =
    \begin{bmatrix}
        p_2 & p_3 & p_1 \\
        p_3 & p_2 & p_1 \\
        p_1 & p_1 & p_4
    \end{bmatrix}.
$$
where $c_{ij}$ is the value of the entries in the block $A_{ij}$.

A block $A_{ij}$ can be expressed as $A_{ij}=c_{ij}\ones_{n_i}\ones_{n_j}\Trans$,
where $(n_1,n_2,n_3)=(R-G,R-G,G)$ and $\ones_n$ denotes a column vector of ones with dimension $n$.

Therefore,
$$
    A_1 = \begin{bmatrix}
        c_{11}\ones_{n_1}\ones_{n_1}\Trans &
        c_{12}\ones_{n_1}\ones_{n_2}\Trans &
        c_{13}\ones_{n_1}\ones_{n_3}\Trans \\
        c_{21}\ones_{n_2}\ones_{n_1}\Trans &
        c_{22}\ones_{n_2}\ones_{n_2}\Trans &
        c_{23}\ones_{n_2}\ones_{n_3}\Trans \\
        c_{31}\ones_{n_3}\ones_{n_1}\Trans &
        c_{32}\ones_{n_3}\ones_{n_2}\Trans &
        c_{33}\ones_{n_3}\ones_{n_3}\Trans \\
    \end{bmatrix}
    =
    \begin{bmatrix}
        \begin{bmatrix}
            c_{11} \ones_{n_1} \\
            c_{12} \ones_{n_2} \\
            c_{13} \ones_{n_3}
        \end{bmatrix}\ones_{n_1}\Trans &
        \begin{bmatrix}
            c_{21} \ones_{n_1} \\
            c_{22} \ones_{n_2} \\
            c_{23} \ones_{n_3}
        \end{bmatrix}\ones_{n_2}\Trans &
        \begin{bmatrix}
            c_{31} \ones_{n_1} \\
            c_{32} \ones_{n_2} \\
            c_{33} \ones_{n_3}
        \end{bmatrix}\ones_{n_3}\Trans
    \end{bmatrix}.
$$

It is evident that any vector in the nullspace of $\ones_{n_i}\Trans$ can be expanded to a vector in the nullspace of $A$, by filling the remaining blocks with zeros. Therefore, $A$ has an eigenvalue of $0$ with multiplicity at least $2R-G-3$.

In order to find the other three eigenvalues, we further consider the matrix $Q$:
$$
    Q = {(q_{ij})}_{3\times 3} =
    \begin{bmatrix}
        (R-G) p_2 & (R-G) p_3 & G p_1 \\
        (R-G) p_3 & (R-G) p_2 & G p_1 \\
        (R-G) p_1 & (R-G) p_1 & G p_4
    \end{bmatrix}.
$$
where $q_{ij}=n_j c_{ij}$.
Let $\lambda$ be an eigenvalue of $Q$, with corresponding eigenvector $\vv$.
From the equation $Q\vv=\lambda \vv$, we have
$$
\sum\limits_{j=1}^{3} c_{ij}n_j v_j = \lambda v_i,
$$
where $v_i$ is the $i$-th element of $\vv$.
Next, we define $\vu=\vv\otimes\begin{bmatrix}
    \ones_{n_1} \\
    \ones_{n_2} \\
    \ones_{n_3}
\end{bmatrix}=\begin{bmatrix}
    v_1\ones_{n_1} \\
    v_2\ones_{n_2} \\
    v_3\ones_{n_3}
\end{bmatrix}$, where $\otimes$ denotes the Kronecker product. Let $\vw=A\vu$.
Then, the $i$-th block of $\vw$, $\vw_i$, satisfies:
$$
    \vw_i = \sum\limits_{j=1}^{3} A_{ij}\vu_j
    = \sum\limits_{j=1}^{3} c_{ij}\ones_{n_i}\ones_{n_j}\Trans v_j\ones_{n_j}
    = \sum\limits_{j=1}^{3} (c_{ij}n_j v_j) \ones_{n_i}
    = \sum\limits_{j=1}^{3} \lambda v_i \ones_{n_i}
    = \sum\limits_{j=1}^{3} \lambda \vu_i,
$$
Thus, $\vu$ is an eigenvector of $A$ with corresponding eigenvalue $\lambda$. From the structure of $\vu$, we observe that $\vu$ is orthogonal to all eigenvectors corresponding to the eigenvalue $0$ discussed above. Finally, we can find the remaining three eigenvalues as below:

\begin{align*}
\begin{cases}
    \lambda_1 &= 1, \\
    \lambda_2 &= \frac{(H-W)(R-G)}{HR} = \left[1-\left(\frac{M}{N+M}\right)^{\rho K}\right]\left[1-\left(\frac{M}{N+M}\right)^{(1-\rho)K}\right], \\
    \lambda_3 &= 0.
\end{cases}
\end{align*}

We observe that $\lambda_2$ lies within the range $(0,1)$ and remains unchanged if $\rho$ is replaced by $1-\rho$, reflecting a symmetry between the masked and unmasked views. The value of $\lambda_2$ is maximized at $\rho = 0.5$, and diminishes as $\rho$ moves away from this midpoint. According to spectral graph theory \cite{chung1997spectral}, a smaller $\lambda_2$ implies a more connected augmentation graph, supporting our analysis in Section \ref{sec:mask-empirical} that connectivity improves when the mask ratio approaches 1.

\subsection{The Label Error}

Denote the unmasked-to-unmasked label error as $\beta=\sum\limits_{x_1,x_1^+} \A(x_1,x_1^+)\ones[y(x_1)\ne y(x_1^+)]$. To evaluate $\beta$, we consider the following four scenarios:

\begin{itemize}
    \item \textbf{One of $x_1$ and $x_1^+$ is non-overlapped and
    the other is overlapped}: \\
    $\A(x_1,x_1^+) = \frac{1}{2R^2} = q_1$.
    In this case, $\Pr[y(x_1)\ne y(x_1^+)] = \eta$.
    \item \textbf{$x_1$ and $x_1^+$ are in the same class and both non-overlapped}: \\
    $\A(x_1,x_1^+) = \frac{2H-W}{4HR^2} = q_2$.
    In this case, $\Pr[y(x_1)\ne y(x_1^+)] = 0$.
    \item \textbf{$x_1$ and $x_1^+$ are in different classes and both non-overlapped}: \\
    $\A(x_1,x_1^+) = \frac{W}{4HR^2} = q_3$.
    In this case, $\Pr[y(x_1)\ne y(x_1^+)] = 1$.
    \item \textbf{Both $x_1$ and $x_1^+$ are overlapped}: \\
    $\A(x_1,x_1^+) = \frac{1}{R^2} = q_4$.
    In this case, $\Pr[y(x_1)\ne y(x_1^+)] = \eta$.
\end{itemize}

Here, $\eta$ should be a value greater than $0$ and considerably smaller than $0.5$. The reason is that in practice, even if a patch possibly belongs to both two classes, its membership to a class should still be on a spectrum, \ie its belonging is not completely indistinguishable; considering that an unmasked view comprises many patches, the indistinguishability, \ie label error should be significantly smaller than the completely stochastic scenario. So, we assume $0<\eta<0.3$ here.

Then we have

\begin{equation*}
    \begin{split}
    \beta &=
        4(R-G)G    \cdot q_1 \cdot \eta +
        2{(R-G)}^2 \cdot q_2 \cdot 0 +
        2{(R-G)}^2 \cdot q_3 \cdot 1 +
        G^2        \cdot q_4 \cdot \eta \\
        &= \frac{8H(R-G)G\eta+2(R-G)^2W+4HG^2\eta}{4HR^2}. \\
    \end{split}
\end{equation*}

Denote the unmasked-to-natural label error as $\alpha=\E_{\bar{x},x_1}\ones[y(x_1)\ne y(\bar{x})]$. We have

\newcommand{\xo}{{x_1\textrm{ is overlapped}}}
\begin{align*}
    \alpha &=
        \E_{\bar{x}} \Pr[y(x_1)\ne y(\bar{x})] \\
        &= \Pr(\xo)[
            \Pr(x_1\in C_2|\xo)\Pr(\bar{x}\in C_1|\xo)\\
            &\qquad\qquad\qquad\qquad\qquad+\Pr(x_1\in C_1|\xo)\Pr(\bar{x}\in C_2|\xo)
        ] \\
        &=
            \frac{G}{R}\left[\frac{1}{2}\cdot\frac{1}{2}+\frac{1}{2}\cdot\frac{1}{2}\right] \\
        &= \frac{G}{2R} \\
        &= {\left(\frac{M}{N+M}\right)}^{(1-m)K}.
\end{align*}

\subsection{Proof of Theorem~\ref{thm:generalization on toy}}
\begin{proof}
As shown in Equation~(\ref{eqn:pred error alpha beta}), we consider $S=32\lambda_2^2 + 32\beta + 16\alpha$ as a component of the upper bound for the downstream classification error that relates to $\rho$.
Let $w=\frac{W}{H}=r^{\rho K}$, $g=\frac{G}{R}=r^{(1-\rho)K}$ and $V=wg=r^K$. Note that $V$ is independent of the mask ratio $\rho$. As stated in Theorem \ref{thm:generalization on toy}, $V$ satisfies $0<V<\frac{1}{12}$.
Then we have

$$
\begin{aligned}
    S
    &= 32\lambda_2^2 + 32\beta + 8\alpha \\
    &= 32{\left[\frac{(H-W)(R-G)}{HR}\right]}^2 + 32\frac{8H(R-G)G\eta+2(R-G)^2W+4HG^2\eta}{4HR^2} + 8\frac{G}{2R} \\
    &= 32g^2w^2-48g^2w+32(1-\eta)g^2-64gw^2+96gw+[64 (\eta-1)+4]g+32w^2-48w+32.
\end{aligned}
$$

Let $S^*(\rho,r,K)=\frac{S-32}{4}$. We now regard $g$ and $w$ as variables that comply with the constraints $V<w<1$ and $wg=V$. By plugging in $wg=V$ we get

$$
\begin{aligned}
    S^* &= 8V^2-12Vg+8(1-\eta)g^2-16Vw+24V+[16(\eta-1)+1]g+8w^2-12w \\
    &= 8(1-\eta)g^2+[16(\eta-1)+1-12V]g+8w^2-(16V+12)w+8V^2+24V.
\end{aligned}
$$

Let $x=2\sqrt{2}\cdot w$, $y=2\sqrt{2(1-\eta)}\cdot w$, $a=\frac{\sqrt{2}(4V+3)}{2}$ and $b=\frac{12V-16\eta+15}{4\sqrt{2(1-\eta)}}$. Then $S^*=(x-a)^2+(y-b)^2+\varPhi(V,\eta)$, where $\varPhi(V,\eta)$ is a function of $\eta$ and $V$ which is independent of $w$ and $g$, and consequently irrelevant to $\rho$. We also note that $xy=8\sqrt{1-\eta}\cdot V:=c$. By plugging this in, we have

$$
    S^* = (x-a)^2+\left(\frac{c}{x}-b\right)^2+\varPhi(V,\eta).
$$

Take the derivative of $S^*$ with respect to $x$:

$$
    \fracpartial{S^*}{x}=2(x-a)-\frac{2c}{x^2}\left(\frac{c}{x}-b\right)
$$

We consider the sign of the derivative at two points.

When $x=\frac{c}{b}$,

$$
\begin{aligned}
    \left.\fracpartial{S^*}{x}\right|_{x=\frac{c}{b}}&=2\left(\frac{c}{b}-a\right) \\
    &= \frac{2}{b}(c-ab) \\
    &= \frac{2}{b}\left[8\sqrt{1-\eta}\cdot V-\frac{(4V+3)\sqrt{2}}{2}\frac{12V-16\eta+15}{4\sqrt{2(1-\eta)}}\right] \\
    &= \frac{2}{b}\cdot\frac{-48V^2-32V+48\eta-45}{8\sqrt{1-\eta}} \\
    &< \frac{2}{b}\cdot\frac{48\cdot 0.3-45}{8\sqrt{1-\eta}} & (V>0,\ \eta<0.3) \\
    &<0.
\end{aligned}
$$

When $x=\sqrt{c}$,

\begin{align}
    \left.\fracpartial{S^*}{x}\right|_{x=\sqrt{c}}&=2\sqrt{c}(b-a) \nonumber \\
    &= 2\sqrt{c}\cdot  \left[\frac{12V-16\eta+15}{4\sqrt{2(1-\eta)}} - \frac{(4V+3)\sqrt{2}}{2}\right] \nonumber \\
    &= \frac{\sqrt{2c}}{4\sqrt{1-\eta}}\cdot  \left[-4V\left(4\sqrt{1-\eta}-3\right)-16\eta-12\sqrt{1-\eta}+15\right] \nonumber \\
    &> \frac{\sqrt{2c}}{4\sqrt{1-\eta}}\cdot  \left[16-16\eta-\frac{40}{3}\sqrt{1-\eta}\right] & \left(V<\frac{1}{12}\right) \nonumber \\
    &= \sqrt{2c}\cdot  \left[16\sqrt{1-\eta}-\frac{40}{3}\right] \nonumber \\
    &> \sqrt{2c}\cdot  \left[16\sqrt{0.7}-\frac{40}{3}\right] & (\eta<0.3) \nonumber \\
    &> 0. \label{eqn:b-a gt 0}
\end{align}

From the signs of the derivative at certain points we know that $S^*$ reaches a local minimum at $x=x^*$, where $\frac{c}{b}<x^*<\sqrt{c}$ (if there are multiple local minima, we take the one that leads to the smallest $S^*$). Also note that $\frac{c}{b}<\sqrt{c}$ holds, since this is equivalent to $c<b^2$ and we have

$$
\begin{aligned}
    c &= 8\sqrt{1-\eta}\cdot V \\
    &< 0.67 & \left(V<\frac{1}{12},\ \eta>0\right) \\
    &< 1.8 \\
    &< \frac{12V-16\eta+15}{4\sqrt{2(1-\eta)}} & \left(V>0,\ 0<\eta<0.3\right) \\
    &= b \\
    &< b^2.
\end{aligned}
$$

$x<\sqrt{c}$ is equivalent to $w=\frac{x}{2\sqrt{2}}<\sqrt{\sqrt{(1-\eta)}V}<\sqrt{V}$; together with $w=r^{\rho K}=V^\rho$ and $V=r^K<1$ we have $\rho>0.5$. $x>\frac{c}{b}$
is equivalent to $w>\frac{c}{2\sqrt{2}\cdot b}$, which means

$$
\begin{aligned}
    w &> \frac{8\sqrt{1-\eta}}{2\sqrt{2}}\cdot\frac{4\sqrt{2(1-\eta)}}{12V-16\eta+15} \cdot V \\
    &= \frac{16(1-\eta)}{12V+15-16\eta}\cdot V \\
    &> \frac{16(1-\eta)}{16-16\eta}\cdot V & \left(V<\frac{1}{12}\right) \\
    &= V.
\end{aligned}
$$

This in turn implies $\rho<1$. Therefore, the local minimum $x^*$ is achieved when $0.5<\rho<1$.

It remains to show that this local minimum is a global minimum when $w>0$, \ie $x>0$. Consider the second derivative of $S^*$ with respect to $x$:

$$
\begin{aligned}
\frac{\partial^2 S^*}{\partial x^2} &= \frac{2 c^{2}}{x^{4}} + \frac{4 c \left(\frac{c}{x} - b\right)}{x^{3}} + 2.
\end{aligned}
$$

Clearly, $\frac{\partial^2 S^*}{\partial x^2}>0$ when $x<\frac{c}{b}$. This implies that $S^*$ does not have a zero point when $x<\frac{c}{b}$, meaning that $x^*$ is a global minimum for $0<x<\sqrt{c}$. For $x>\sqrt{c}$, we consider the symmetry of the hyperbola $xy=C$. Let $x_1<\sqrt{c}$, then $\frac{c}{x_1}>\sqrt{c}$. We have
\begin{align}
    \left.S^*\right|_{x=x_1}-\left.S^*\right|_{x=\frac{c}{x_1}} &= (x_1-a)^2+\left(\frac{c}{x_1}-b\right)^2-\left(\frac{c}{x_1}-a\right)^2-(x_1-b)^2 \nonumber \\
    &= \left(x_1+\frac{c}{x_1}-2a\right)\left(x_1-\frac{c}{x_1}\right)-\left(x_1+\frac{c}{x_1}-2b\right)\left(x_1-\frac{c}{x_1}\right) \nonumber \\
    &= 2\left(x_1-\frac{c}{x_1}\right)(b-a) \nonumber \\
    &< 0. \qquad\qquad \left(x_1<\sqrt{c}<\frac{c}{x_1};\ b-a>0,\ \text{from inequality}\ \ref{eqn:b-a gt 0}\right) \label{eqn:hyperbola-symmetry}
\end{align}

Now, suppose there exists a local minimum $\hat{x}>\sqrt{c}$. From inequality (\ref{eqn:hyperbola-symmetry}), we know that $\left.S^*\right|_{x=\hat{x}}>\left.S^*\right|_{x=\frac{c}{\hat{x}}}$. Since $x^*$ is a global minimum for $0<x<\sqrt{c}$ and $0<\frac{c}{\hat{x}}<\sqrt{c}$, we have $\left.S^*\right|_{x=\frac{c}{\hat{x}}}>\left.S^*\right|_{x=x^*}$. Thus, any local minimum $\hat{x}>\sqrt{c}$ results in greater $S^*$ than $x^*$ does.

In conclusion, $x^*$ is the global minimum for $x>0$. This completes the proof of Theorem \ref{thm:generalization on toy}.
\end{proof}

\section{Visualization of the Augmentation Graph}
\label{sec:vis of augmentation graph}

To provide a better visualization of the augmentation graph, we randomly selected 6 image pairs from the ImageNet data set: 3 intra-class pairs  and 3 inter-class pairs. For each image, 500 random masks were applied with a mask ratio of 75\%. A Vision Transformer, pretrained using Masked Autoencoder (MAE) on ImageNet, was then utilized to estimate the similarity between each pair of unmasked views. This estimation was based on the cosine similarity of the extracted features. From these, the top three unmasked view pairs, as determined by their similarity scores, were selected for visualization. The results are presented in Figure~\ref{fig:vismae}.

Clearly, unmasked views within intra-class image pairs exhibit significantly higher similarity compared to those in inter-class pairs, which also leads to larger weights for the edges in the augmentation graph. The similarity is illustrated by the presence of common iconic patches, such as the wheels in Figure~\ref{fig:vismae_same_1}, the fur in Figure~\ref{fig:vismae_same_2} and the piles of oranges in Figure~\ref{fig:vismae_same_3}. Conversely, the unmasked regions of inter-class image pairs appear more heterogeneous and lack a consistent pattern, underscoring the fundamental visual differences between distinct classes.

\begin{figure}
    \centering
    \subfigure[]{
        \includegraphics[width=0.23\textwidth]{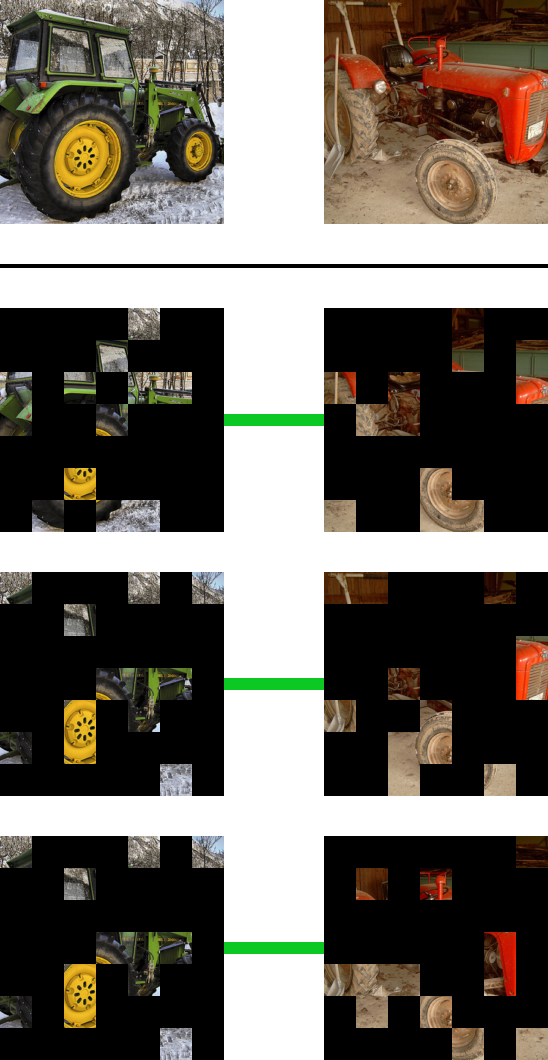}
        \label{fig:vismae_same_1}
    }
    \hspace{0.08\textwidth}
    \subfigure[]{
        \includegraphics[width=0.23\textwidth]{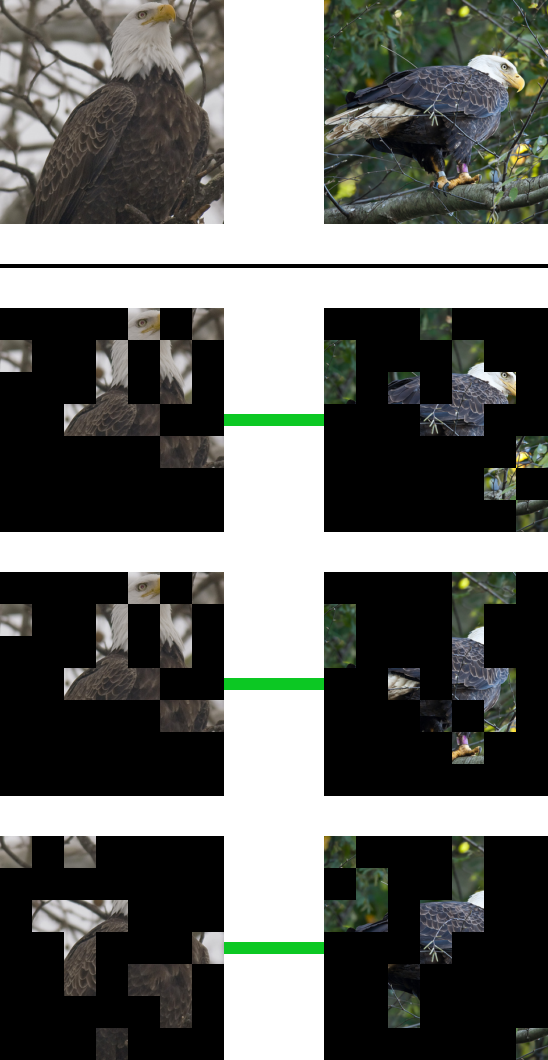}
        \label{fig:vismae_same_2}
    }
    \hspace{0.08\textwidth}
    \subfigure[]{
        \includegraphics[width=0.23\textwidth]{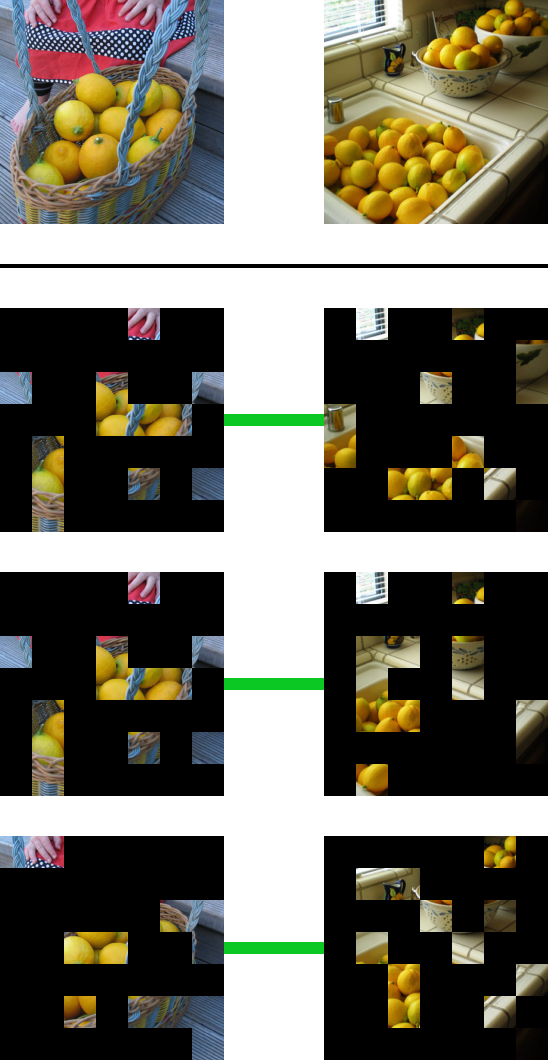}
        \label{fig:vismae_same_3}
    }

    \subfigure[]{
        \includegraphics[width=0.23\textwidth]{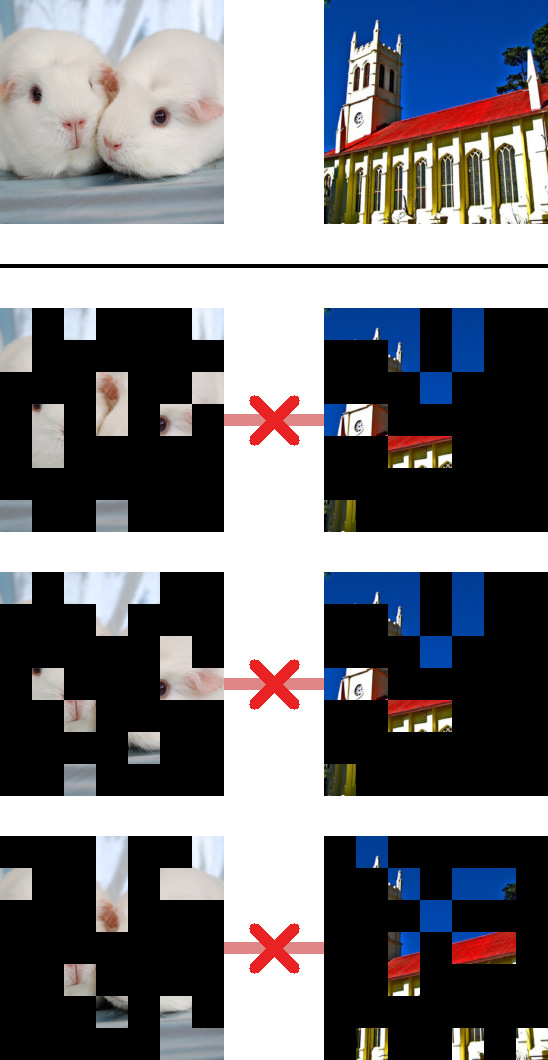}
        \label{fig:vismae_diff_1}
    }
    \hspace{0.08\textwidth}
    \subfigure[]{
        \includegraphics[width=0.23\textwidth]{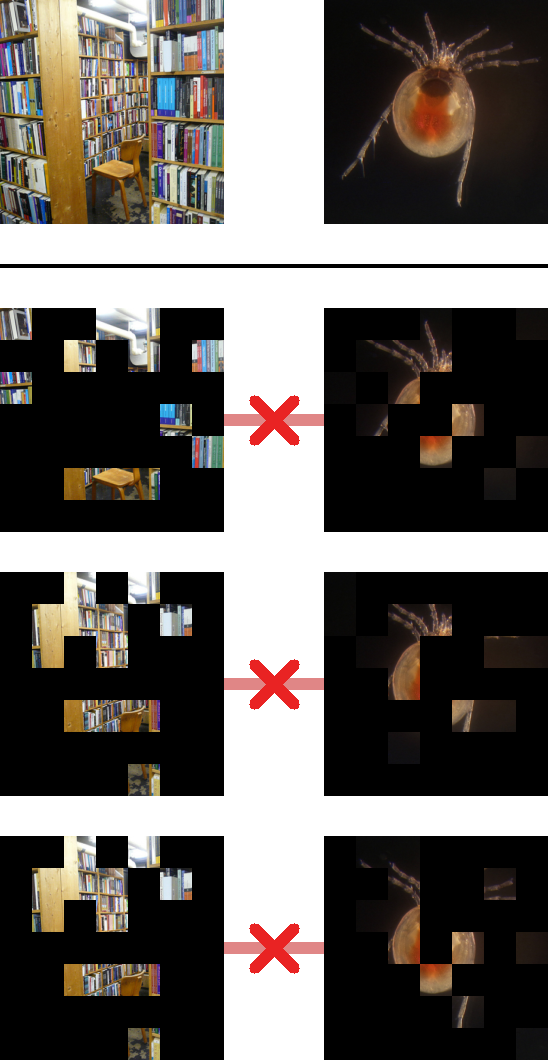}
        \label{fig:vismae_diff_2}
    }
    \hspace{0.08\textwidth}
    \subfigure[]{
        \includegraphics[width=0.23\textwidth]{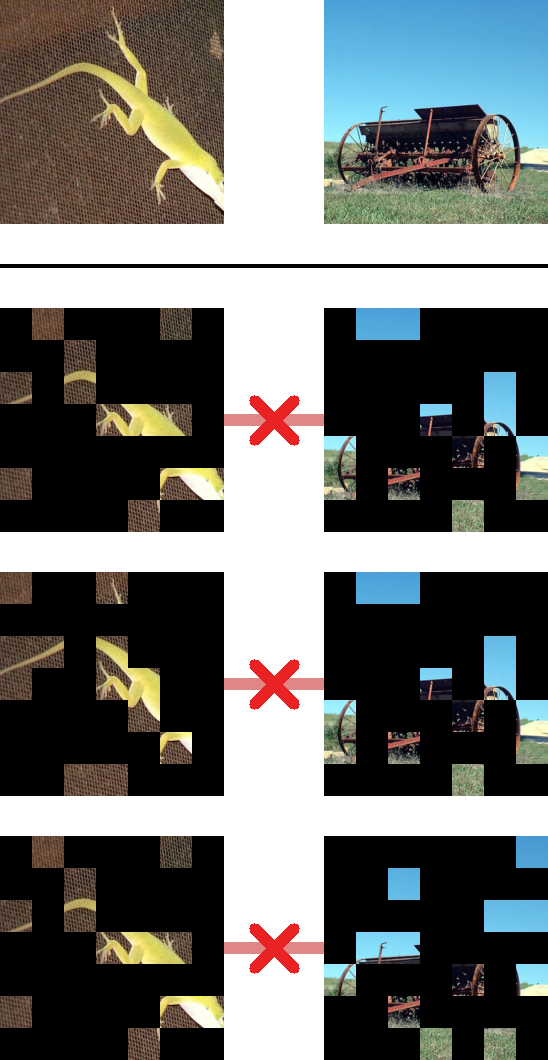}
        \label{fig:vismae_diff_3}
    }

    \caption{Illustration of edges in an augmentation graph between intra-class images and inter-class images. The first row shows that intra-class image pairs tend to have fairly overlapped unmasked views, while the second row demonstrates that even the most similar unmasked views of inter-class image pairs are non-overlapped.}
    \label{fig:vismae}
\end{figure}

\newpage
\bibliography{sample}

\end{document}